\documentclass[11pt]{article}

\usepackage{amsmath}
\usepackage{amssymb}
\usepackage{fullpage}
\usepackage{color}
\usepackage{graphicx}
\usepackage{epsfig}
\usepackage{amsthm}
\usepackage{latexsym}
\usepackage{amssymb}
\usepackage{amsmath}
\usepackage{verbatim}
\usepackage{graphicx}
\usepackage{epsfig}
\usepackage{subcaption}
\usepackage{cite}
\usepackage{url}
\usepackage{tabularx}
\usepackage{algorithmic}
\usepackage{algorithm}
\usepackage{enumitem}
\usepackage{mathtools}
\usepackage{bbm}

\usepackage[usenames,svgnames,xcdraw,table]{xcolor}
\definecolor{DarkBlue}{rgb}{0.1,0.1,0.5}
\definecolor{DarkGreen}{rgb}{0.1,0.5,0.1}
\usepackage[backref=page]{hyperref}
\hypersetup{
     colorlinks   = true,
     linkcolor    = DarkBlue, 
     urlcolor     = DarkBlue, 
	 citecolor    = DarkGreen 
}
\renewcommand*{\backref}[1]{}
\renewcommand*{\backrefalt}[4]{%
    \ifcase #1 (Not cited.)%
    \or        (Cited on page~#2)%
    \else      (Cited on pages~#2)%
    \fi}
\usepackage[capitalise,noabbrev]{cleveref}

\usepackage{thmtools,thm-restate}

\newtheorem{theorem}{Theorem}

\newtheorem{proposition}{Proposition}

\newtheorem{corollary}{Corollary}
\newtheorem{definition}{Definition}

\newcommand{\E}{\mathbb{E}}

\begin{document}
\title{{\bfseries Quantifying Infra-Marginality and Its Trade-off with\\ Group Fairness}}
\author{\textbf{Arpita Biswas}\\ Indian Institute of Science$\quad$\\
arpitab@iisc.ac.in
\and
\textbf{Siddharth Barman}\\ Indian Institute of Science\\ barman@iisc.ac.in
\and \textbf{Amit Deshpande}\\ Microsoft Research India\\ amitdesh@microsoft.com
\and \textbf{Amit Sharma}\\ Microsoft Research India\\amshar@microsoft.com
}
\date{}
\maketitle

\begin{abstract}
In critical decision-making scenarios, optimizing accuracy can lead to a biased classifier, hence past work recommends enforcing group-based  fairness metrics in addition to maximizing accuracy. However, doing so exposes the classifier to another kind of bias called \emph{infra-marginality}. This refers to individual-level bias where some individuals/subgroups can be worse off than under simply optimizing for accuracy. For instance, a classifier implementing  race-based parity may significantly disadvantage women of the advantaged race. To quantify this bias, we propose a general notion of $\eta$-infra-marginality that can be used to evaluate the extent of this bias. We prove theoretically that, unlike other fairness metrics,  infra-marginality does not have a trade-off with accuracy: high accuracy directly leads to low infra-marginality. This observation is confirmed through empirical analysis on multiple simulated and real-world datasets. Further, we find that maximizing group fairness often increases infra-marginality, suggesting the consideration of both group-level fairness and individual-level infra-marginality. 
However, measuring infra-marginality requires knowledge of the true distribution of individual-level outcomes correctly and explicitly. We propose a practical method to measure infra-marginality, and a simple algorithm to maximize group-wise accuracy and avoid infra-marginality.
\end{abstract}

\section{Introduction}
Consider a machine learning algorithm being used to make decisions in societally critical domains such as healthcare \cite{irene2018why,goodman2018machine}, education \cite{tierney2013fairness}, criminal justice \cite{angwin2016machine,berk2017fairness}, policing~\cite{simoiu2017problem,goel2017combatting} or finance \cite{furletti2002overview}. Since data on past decisions may include historical or societal biases, it is generally believed that optimizing accuracy can result in an algorithm that sustains the same biases and thus is unfair to underprivileged groups \cite{romei2014multidisciplinary,barocas2016big,chouldechova2018frontiers}. Theoretically too, it has been shown that it is not possible to have calibration and 
fairness simultaneously~\cite{kleinberg2017inherent} and several other results show that achieving multiple fairness constraints simultaneously is infeasible~\cite{chouldechova2017fair,corbett2017algorithmic}.

Given this tradeoff between accuracy and fairness metrics, algorithms typically aim to maximize accuracy while also satisfying different notions of fairness, such as disparate impact~\cite{kamiran2012data,feldman2015certifying,zafar2017afairness},  statistical parity~\cite{kamishima2012fairness,zemel2013learning,corbett2017algorithmic}, 
equalized odds~\cite{hardt2016equality,kleinberg2017inherent,woodworth2017learning}. Most of these fairness notions, however, enforce group-level constraints on pre-specified \emph{protected} attributes and give no guarantees on fairness with respect to other sensitive attributes.  
That is, it is difficult to generalize these group-level constraints to be fair with respect to multiple sensitive attributes which can lead to biases when individuals may belong to multiple historically disadvantaged groups.  As an example, a classifier that is constrained to be fair on race may end up introducing discrimination against women of the advantaged race, and vice-versa. Moreover,  being group-wise constraints, they provide no guarantees on fairness for individuals' outcomes within these groups. 

In this paper, therefore, we consider a different notion of bias, \emph{infra-marginality}~\cite{simoiu2017problem, corbett2018measure}, that can handle multiple sensitive attributes simultaneously and enforces individual-level rather than group-level adjustments. The idea behind this individual-level fairness is that individuals with the same probability of an outcome (say, the same risk probability) receive the same decision, irrespective of their sensitive attributes. Any deviation from this ideal necessarily leads to misclassifying low-risk members as high risk and high-risk members as low risk, which in turn may harm the members of all groups. For instance, consider the medical domain where doctors assess the severity of a person's illness (risk probability) and prioritize treatment accordingly. It may be acceptable to prioritize patients based on a given reliable estimate of a patient's risk, but any deviation from this rule due to a group-fairness constraint on a particular demographic, may deprive high-risk people from treatment, some of whom in turn may belong to another disadvantaged demographic. Hence, these group-level adjustments often introduce unintended biases and can be marginally \emph{unfair} at an individual-level; thus the name ``infra-marginality'' (see Section~\ref{sec:infra} for a definition).  Conceptually, therefore, mitigating infra-marginality has a notable advantage: the fairness of a decision is based directly on the underlying  outcome probability of each individual, rather than post hoc group-level adjustments or constraints. 

To remove infra-marginality, Simoiu et al.~\cite{simoiu2017problem} propose taking decisions using a single threshold on the true outcome probability, whenever the outcome probability is known.\footnote{In the above-mentioned example from the medical domain, this corresponds to selecting a patient (prioritizing treatment) if and only if the risk (outcome) probability of the patient is above a threshold.} Such a single-threshold classifier implements the high-level idea that legislation should apply equally to all individuals and not be based on group identity. A set of decisions (equivalently, a classifier) that conforms to this ideal is said to have zero infra-marginality (and, thus, single-threshold fair)~\cite{simoiu2017problem, corbett2018measure}. 

Extending past work on infra-marginality~\cite{simoiu2017problem}, our main contribution is to quantify the notion and propose a generalized version of infra-marginality which we call $\eta$-infra-marginality. It is relevant to note that Simoiu et al.~\cite{simoiu2017problem} essentially consider infra-marginality as a ``binary'' property--either a classifier suffers from infra-marginality or it does not. Extending this construct, the current work defines the \emph{degree of infra-marginality}. Furthermore, using this definition, we prove theoretically that high accuracy of a classifier directly leads to low infra-marginality. 
These results also hold group-wise: high accuracy, per group, means low group-wise infra-marginality.

Results on simulated and real-world datasets confirm this observation. First, we construct a set of simulated datasets with varying distributions of outcome per demographic group. Under all simulations, we find that training a machine learning model to maximize accuracy with respect to the outcome also yields low infra-marginality. Specifically, infra-marginality is at most $\delta$ distance away from the classification error. Here, parameter $\delta$ depends solely on the underlying data set, i.e., it is fixed for the given problem instance and is independent of any classifier one considers. In addition, when we focus on individuals from a sensitive group separately, we find that maximizing group-wise accuracy results in low infra-marginality for both groups. Second, we consider two datasets that are widely used for studying algorithmic fairness, namely \texttt{Adult-Income} and \texttt{Medical-Expense} datasets, and demonstrate the connection between maximizing accuracy and lowering infra-marginality. Since there is no ground-truth for individual outcome probabilities, we create an evaluation testbed by developing classifiers with subsets of features and benchmarking infra-marginality with respect to the outcome probabilities learnt using all features. This benchmark serves as an approximation of the true outcome probability. We find that even when a classifier is not trained on true outcome risk, we find a similar positive trend between high accuracy and low infra-marginality.       

The close connection of infra-marginality to accuracy also illustrates the inherent tradeoff with group fairness. Since group-wise fairness constraints necessarily reduce accuracy, they exacerbate infra-marginality. 
Using the meta-fair algorithm~\cite{celis2018classification} for training group-fair classifiers, we find that increasing the weight on a group-wise fairness constraint (such as demographic parity or equal false discovery rates) increases infra-marginality. This result is consistent for both simulated and real-world datasets: whenever constraints on accuracy optimization lead to an increase in a group-wise fairness metric, they also lead to an increase in infra-marginality bias. We argue these results present a difficult choice for ensuring fairness: group-wise constraints may be blind to fairness on other unobserved groups, while lowering infra-marginality increases individual-level fairness but may lead to group-level unfairness.   
Thus, as a practical measure, we propose maximizing group-wise accuracy for lowering individual-level bias whenever the true outcome probability is measurable.

While these results point to the importance of considering infra-marginality, bounding infra-marginality in practice is a challenge because it depends on the knowledge of a true outcome probability for individuals. This is often not available, especially when the true probabilities of individuals are not known and are required to be learnt using datasets which often contain only one-sided information (for example, recidivism outcomes of only those individuals who were granted bail or loan repayment outcome of only those who were granted a loan). In such cases, we may not have correct estimates of risk probabilities of the underlying population and hence infra-marginality remains a theoretical concept. 

To be useful in practice, we need two conditions: first, that objective measures of the outcomes are available and, second, that the available data is sampled uniformly from the target population of interest (and not one-sided, or based on biased decisions). Happily, outcomes of interest are available in many decision-making scenarios where the outcomes can be objectively recorded. For example, in the medical domain, outcomes are often categorical and objective (e.g., cured of a disease or not). Objective outcomes also occur in security-related decisions such as searching for  a weapon in a vehicle or screening passengers at an airport for prohibited goods. 
In all these cases, the outcome of interest is measurable and is not affected by underlying biases, unlike subjective outcomes such as success in school or work, awarding loans, etc.
  The second condition is more stringent and requires that we observe the outcomes for a \emph{representative} sample of the underlying population and not be biased by past decisions. This assumption can be satisfied either by utilizing additional knowledge on the decision-making process or by actively changing decision-making. In Section~\ref{sec:conclude}, we discuss potential approaches such as by assuming that a subset of decisions are calibrated to the true outcome risk~\cite{goel2017combatting,simoiu2017problem,pierson2018fast}, by having a random sample of outcome data, by randomizing the decision for a fraction of users, or by advanced strategies such as contextual bandits~\cite{agarwal2017bandits}.\\  

Overall, our results point to the importance of data design over statistical adjustments in achieving fairness in algorithmic decision-making and the limitations of depending on an available dataset for ensuring fairness. Rather than post hoc adjustments to steer an algorithm towards a chosen outcome, or introducing external constraints, it is worthwhile to consider obtaining accurate and unbiased outcome measurements. Under these conditions, our work proposes that the practice of optimizing for group-wise accuracy also leads to low bias (low infra-marginality), and allows improvements in machine learning to directly be applicable in improving fairness.\\

To summarize, we make the following contributions:
\begin{itemize}
      \item Our conceptual contribution is to develop a quantitative measure to characterize the problem of infra-marginality. Extending past work~\cite{simoiu2017problem, corbett2018measure} on infra-marginality in algorithmic fairness, we propose a general $\eta$-infra-marginality formulation that enables a measure of discrimination under the definition of {single-threshold fairness} (Section~\ref{sec:infra}).
      \item Second, we show that infra-marginality has a striking property: Within an additive margin, the more accurate the classifier, lower is its degree of infra-marginality (Section~\ref{section:accuracy-infra}).       We provide a rigorous proof of this claim in Theorem~\ref{theorem:main}. This result asserts that, for well-behaved data sets and given a classifier threshold $\tau$, higher accuracy points to lower infra-marginality. Also, complementarily, an inaccurate classifier will necessarily induce infra-marginality, to a certain degree. 
Moreover, this result holds group-wise (Corollary~\ref{cor:groups}), implying that higher group-wise accuracy reduces the infra-marginality problem per group. We also provide two propositions that identify relevant settings wherein Theorem~\ref{theorem:main} can be applied.
\item Third, when the classifier threshold $\tau$ is not fixed,  we provide an algorithm for learning a classifier that optimizes accuracy subject to infra-marginality constraints, assuming that the true outcome probabilities are available (Section~\ref{section:opt-class}). We show that the problem reduces to applying a linear constraint and hence we can efficiently find an optimal solution. This is notably in contrast to prominent fairness metrics, which introduce non-convex (fairness) constraints in the underlying learning algorithms. 
\end{itemize}

\section{Background and Related Work}
Datasets that collect past decisions on individuals and their resultant outcomes often reflect prevailing societal biases~\cite{barocas2016big}.  These biases correspond to discrimination in decision-making, in recording outcomes, or both. During decision-making, individuals from a specific group may be less preferred for a favorable decision, thus reducing their chances of appearing in a dataset with favorable outcomes. This is especially the case when data recording is one-sided: we observe the outcomes only for those who received a favorable decision, such as being awarded a job, loan or bail. In addition, the discrimination can also occur when recording outcomes, e.g.,  when measuring hard-to-measure outcomes such as defining ``successful'' job candidates, employees or students. Due to these biases in dataset collection, building a decision-making algorithm by maximizing accuracy on the dataset perpetuates the bias in selection of individuals or the measurement of a desirable outcome.

To counter this bias, various group-wise statistical constraints have been proposed for a decision classifier that stipulate equitable treatment of different demographic groups. \emph{Demographic or statistical parity} says that a fair algorithmic prediction should be independent of the sensitive attribute and thus each demographic group should have the same fraction of favorable decisions \cite{kamishima2012fairness,zemel2013learning,corbett2017algorithmic}. \emph{Equalized odds} \cite{hardt2016equality,kleinberg2017inherent,woodworth2017learning} says that for a fair algorithmic prediction, the true positive rates and the false positive rates on different sensitive demographics should be as equal as possible. Combined with the accuracy objective, equalized odds imply similar, high accuracy for every sensitive demographic. \emph{Disparate impact} \cite{kamiran2012data,feldman2015certifying,zafar2017afairness} refers to the impact of policies that affect one sensitive demographic more than another, even though the rules applied are formally neutral. Mathematically, the probability of being positively classified should be the same for different sensitive demographics. Chouldechova and Roth~\cite{chouldechova2018frontiers} along with Friedler et al.~\cite{friedler2018comparative} provide a survey of various fairness notions and algorithms to incorporate them. 
However, in practice, the output of these group-fair algorithms may show worse predictions among all the groups; for example, while trying to ensure \textit{equal} false positive rates among two demographic groups, the predictions may end up increasing false positive rates for both the groups.  

The difficulty of ensuring fairness for different groups suggests an alternative definition of fairness based on individual-level constraints. Individual fairness \cite{dwork2012fairness,zemel2013learning} propose that similar individuals should be treated as similarly as possible. Thus, rather than stratifying users by pre-specified groups, individual-level fairness constraints use available data (e.g., demographics) to define a similarity measure between individuals and then enforce equitable treatment for all similar individuals. The question then is, how to define a suitable similarity measure? Rather than choosing variables for defining similarity on, \cite{simoiu2017problem} propose defining similarity between two individuals based on the true probability of outcome for the two individuals. That is, people with the same underlying probability of a favorable outcome are the most similar to each other. They define that a classifier has \emph{infra-marginality} bias if people with the same probability of the outcome are given different decisions~\cite{simoiu2017problem,pierson2018fast,corbett2018measure}.   Crucially, this definition depends on a measurement of true outcome probability, but does not restrict analysis to a few pre-specified groups.  

Given these considerations, different fairness notions may be suitable for different settings. When all important marginalized groups are defined (e.g. by law), then group-fairness constraints are more relevant. When true outcome probabilities are known (such as in randomized search/inspection decisions in security applications), then infra-marginality constraints become a better alternative.  In this paper, we focus on the latter and describe the value of considering infra-marginality in fairness decisions. In Section 3, we extend past work to propose a general notion of infra-marginality and prove our main result on the connection between accuracy and low infra-marginality. In Section 4, we provide extensive empirical results showing the inherent tradeoff between typical group fairness constraints and infra-marginality. 

\section{Defining Infra-marginality}
\label{sec:infra}
In this section, we define and characterize the degree of infra-marginality. This measure quantifies an extent to which a classifier violates the notion single-threshold fairness, i.e., it quantifies the problem of infra-marginality identified by Simoiu et al.~\cite{simoiu2017problem}.
 
\subsection{Problem Setup}
We consider a binary classification problem over a set of instances, $\mathcal{X}$, wherein label $1$ denotes the \emph{positive} class and label $0$ denotes the \emph{negative} class, e.g., in the search-for-a-contraband setup, class label $1$ would indicate that the individual (instance) is in possession of a contraband and, complementarily, a label $0$ would correspond to the case in which the individual is not carrying a contraband. Conforming to the standard framework used in binary classification, we will assume that feature vectors (equivalently, data points) $x$ are drawn from the data set $\mathcal{X}$ via a feature distribution. Furthermore, for an instance $x \in \mathcal{X}$, let $p^*_x$ be the inherent probability of being in the positive class; in the previous example, $p^*_x$ is the endowed probability that an individual $x$ is carrying the contraband. 

Under \textit{single-threshold fairness}, a set of (binary) decisions are deemed to be fair if and only if they are obtained by applying a single threshold on the $p^*$s of the instances  (irrespective of the instance's sensitive attributes, such as race, ethnicity, and gender). In the above-mentioned contraband example, this corresponds to searching an individual $x$ if and only if its $p^*_x$ is above a (universally fixed) threshold $\tau$. 

\subsection{Degree of Infra-Marginality}
For ease of presentation, we will identify each data point with its feature vector $x$ and use $y^*_x \in \{0, 1\}$ to denote the (outcome) label of the data point $x$. As mentioned before, $p^*_x$ is the outcome probability of each $x \in \mathcal{X}$. Therefore, for every $x \in \mathcal{X}$, the binary label  $y^*_x$ is equal to one with probability $p^*_x$ and, otherwise (with probability $1-p^*_x$), we have $y^*_x = 0$. The standard classification exercise corresponds to learning a classier that optimizes accuracy with respect to the $y^*$ labels.

For a binary classifier ${C}: \mathcal{X} \mapsto \{0 , 1\}$, we will use $\alpha^C$ to denote its accuracy with respect to the labels $y^*_x \sim {\rm Bernoulli} (p^*_x)$, i.e., $\alpha^C \coloneqq \E_{(x, y^*_x)} \left[ \mathbbm{1}_{C(x) = y^*_x} \right]$, here the expectation is with respect to the underlying feature distribution\footnote{That is, the random sample $(x, y^*_x)$ is drawn from an underlying (feature, label) joint distribution and, conditioned on a feature $x$, we have $y^*_x \sim {\rm Bernoulli} (p^*_x)$.}  and $\mathbbm{1}_{C(x) = y^*_x}$ is the indicator random variable which denotes whether the output of the classifier, $C(x)$, is equal to $y^*_x$, or not. Note that high accuracy simply implies a high value of $\alpha^C  \in [0,1]$.

To formally address single-threshold fairness, we define a (deterministic) label, $f^*_x \in \{0, 1\}$, for each data point $x \in \mathcal{X}$. Semantically, $f^*_x$ is the (binary) outcome of an absolutely fair (in the single-threshold sense) decision applied to $x$.  This, by definition, means that $f^*$s are obtained by applying the same (fixed) threshold $\tau$ on the $p^*$s across all instances: for a fixed fairness threshold $\tau \in [0,1]$, we have $f^*_x = 1$ if $p^*_x > \tau$, else if $p^*_x \leq \tau$ then $f^*_x = 0$. 
With this notation in hand we define the central construct of the present work. 
\begin{definition}[Infra-Marginality of a Classifier]
With respect to a given threshold $\tau \in [0,1]$, the degree of infra-marginality, $\eta^C$, of a classifier $C: \mathcal{X} \mapsto \{0,1\}$ is defined as
\begin{align}
\eta^C \coloneqq \E_{x} \left[ \left|C(x) - f^*_x \right| \right].\label{eq:infra}
\end{align}
Here, the expectation is with respect to the feature distribution over the data set $\mathcal{X}$. In addition, for each $x \in \mathcal{X}$, the label $f^*_x = 1$ if $p^*_x > \tau$, otherwise $f^*_x = 0$. 
\end{definition}
A classifier $C$ conforms to the ideal of single-threshold fairness iff $\eta^C = 0$. Furthermore, the quantity $\eta^C$ can be interpreted as the extent to which the classifier's outputs (i.e., $C(x)$s) differ from the ideal labels (i.e., from the single-threshold benchmarks) $(f^*_x)_{x \in \mathcal{X} }$. Indeed, smaller the value of $\eta^C$ the lower is $C$'s infra-marginality. \\

We will use $\mathcal{D}_{p^*}$ and $\mathcal{D}_{f^*}$, respectively, to denote the distributional form of the collection of generative probabilities $(p^*_x)_{x \in \mathcal{X}}$ and labels $( f^*_x )_{x \in \mathcal{X}}$. Formally, $\mathcal{D}_{f^*}$ is a discrete distribution wherein a probability mass of $q$ is placed on $1$ and probability mass $1-q$ is placed on $0$; here, $q$ is the fraction of data points with the property that $f^*_x = 1$, $q := \frac{1}{|\mathcal{X}|} |\{ x \in \mathcal{X} \mid f^*_x = 1 \}|$. Similarly, the cumulative distribution function (cdf) of $\mathcal{D}_{p^*}$ is $F_{\mathcal{D}_{p^*}} (t) := \frac{1}{|\mathcal{X}|} |\{ x \in \mathcal{X} \mid p^*_x \leq t \}|$.
 
Note that the distributions $\mathcal{D}_{p^*}$ and $\mathcal{D}_{f^*}$ are supported on the distinct $p^*_x \in [0,1]$ values (across $x \in \mathcal{X}$) along with $0$ and $1$.   Write $\| \mathcal{D}_{p^*} - \mathcal{D}_{f^*} \|_1$ to denote the \emph{weighted} $\ell_1$ distance between the two distributions; specifically, the $\ell_1$ distance here is computed by normalizing (i.e., weighing) with respect to the feature distribution. Formally,\footnote{Note that if the feature distribution picks data points uniformly at random from $\mathcal{X}$, then this equation simply implies that the $\ell_1$ distance between the two distributions is equal to the average of the differences between the $p^*_x$ and $f^*_x$ values.} 
\begin{align}
\| \mathcal{D}_{p^*} - \mathcal{D}_{f^*} \|_1 & =  \E_x \left[ | p^*_x - f^*_x | \right] \label{eq:ell1}
\end{align}
 
We will also refer to $\| \mathcal{D}_{p^*} - \mathcal{D}_{f^*} \|_1$ as the $\ell_1$ distance between the generative probabilities $(p^*_x)_{x \in \mathcal{X}}$ and labels $( f^*_x )_{x \in \mathcal{X}}$. Note that this distance is a property of the data set and the underlying feature distribution--it is fixed for the given problem instance and is independent of any classifier one considers. 

\section{Accuracy and Infra-marginality}
\label{section:accuracy-infra}
The key result of this section is the following theorem which establishes that as long as the $\ell_1$ distance between $(p^*_x)_{x \in \mathcal{X}}$ and $( f^*_x )_{x \in \mathcal{X}}$ is small, an accurate classifier will also have low infra-marginality. Complementary, under small $\ell_1$ distance, an inaccurate classifier will necessarily induce high infra-marginality. 

\subsection{High accuracy, low infra-marginality}
\begin{theorem}
\label{theorem:main}
For a data set $\mathcal{X}$, let $\delta$ be the $\ell_1$ distance between the generative (outcome) probabilities $(p^*_x)_{x \in \mathcal{X}}$ and the (single-threshold) labels $( f^*_x )_{x \in \mathcal{X}}$. Then, for any binary classifies $C: \mathcal{X} \mapsto \{0, 1 \}$ with accuracy $\alpha^C$, the degree of infra-marginality satisfies 
\begin{align*} \left(1 - \alpha^C \right) - \delta \leq \eta^C \leq \left(1 - \alpha^C \right) +  \delta.
\end{align*}
\end{theorem}
\begin{proof}
 
Using the definition of the accuracy, $\alpha^C$, of classifier $C$, we get 
\begin{align}
\alpha^C & =  \E_{(x, y^*_x)} \left[ \mathbbm{1}_{C(x) = y^*_x} \right] =   \E_{(x, y^*_x)} \left[ 1 - |C(x) - y^*_x| \right] \nonumber \\
& = 1 -  \E_{(x, y^*_x)} \left[ |C(x) - y^*_x| \right] \label{eq:accuracy}
\end{align}
Analogously,  the degree of infra-marginality, $\eta^C$, can be expressed as
\begin{align}
\eta^C & = \E_x \left[ \left|C(x) - f^*_x \right| \right] =  \E_{(x, y^*_x)} \left[ \left|C(x) - f^*_x \right| \right]  \label{eq:fairness}
\end{align}
Summing (\ref{eq:accuracy}) and (\ref{eq:fairness}) gives us $\eta^C + \alpha^C =  1 + \E \left[ \left|C(x) - f^*_x \right| - \left|C(x) - y^*_x \right| \right]$.

Subtracting one from both the sides of the previous equality and considering absolute values, we obtain 
\begin{align}
\left| \eta^C - \left( 1 - \alpha^C \right) \right| & = \bigl\lvert  \ \E \left[ \left| C(x) - f^*_x \right| - \left|C(x) - y^*_x \right| \right] \ \bigr\rvert \nonumber \\
& \leq \E \left[\bigl\lvert \  \left|C(x) - f^*_x \right| - \left|C(x) - y^*_x \right| \ \bigr\rvert \right] \tag{Jensen's inequality} \nonumber \\
& = \E \left[ \left| y^*_x - f^*_x \right| \right] \label{eq:abs-diff}
\end{align}
The last equation follows from the triangle inequality, $|C(x) - f^*_x| \leq |C(x) - y^*_x | + |y^*_x - f^*_x|$ and $|C(x) - y^*_x| \leq |C(x) - f^*_x | + |f^*_x - y^*_x|$. 

Recall that, for each $x$, we have $y^*_x \sim {\rm Bernoulli} (p^*_x)$ and the label $f^*_x$ can be either zero or one. In both of these cases, we have $\E_{y^*_x \sim {\rm Bernoulli} (p^*_x)} \left[ \left| y^*_x - f^*_x \right| \right] = \left| p^*_x - f^*_x \right|$. Therefore, the desired bound holds 
\begin{align*}
\left| \eta^C - \left( 1 - \alpha^C \right) \right| & \leq  \E_{(x, y^*_x)} \left[ \left| y^*_x - f^*_x \right| \right] \tag{via  \eqref{eq:abs-diff}} \\
& = \E_x \left[ \left| p^*_x - f^*_x \right| \right] \\
& = \delta \tag{via \eqref{eq:ell1}}
\end{align*}
\end{proof}

This result shows that, for any classifier, high accuracy points to low infra-marginality. Notably this result applies group-wise: this ``inverse'' connection holds even if we consider accuracy and infra-marginality separately for different subsets (groups) of the data set $\mathcal{X}$. Specifically, Theorem~\ref{theorem:main} can be applied, as is, to obtain Corollary~\ref{cor:groups}. Here, say the set of data points is (exogenously) partitioned into two groups (i.e., disjoint subsets) $\mathcal{X}_1$ and $\mathcal{X}_2$. Also, let the outcome probabilities of the two sets be $(p^*_x)_{x \in {\mathcal{X}_1}}$ and $(q^*_y)_{y \in \mathcal{X}_2}$, respectively, along with the single-threshold labels $( f^*_x )_{x \in \mathcal{X}_1}$ and $( g^*_x )_{y \in \mathcal{X}_2}$. Let $\delta_1$ ($\delta_2$) denote the $\ell_1$ distance between $(p^*_x)_{x \in {\mathcal{X}_1}}$ and $( f^*_x )_{x \in \mathcal{X}_1}$ ($(q^*_y)_{y \in \mathcal{X}_2}$ and $( g^*_x )_{y \in \mathcal{X}_2}$). Following the above-mentioned notational conventions, the accuracy and infra-marginality measures of a classifier $C: \mathcal{X} \mapsto \{0, 1\}$ for group $i \in \{1,2\}$ are $\alpha^C_{i}$ and $\eta^C_{i}$, respectively. With this notation in hand, we have the following group-wise guarantee.

\begin{corollary}
\label{cor:groups}
For a data set $\mathcal{X}$, comprised of two groups $\mathcal{X}_1$ and $\mathcal{X}_2$, and any classifier $C: \mathcal{X} \mapsto \{0, 1\}$, the degree of infra-marginality, in each group, satisfies the following bounds: $\left| \eta^C_1 -  \left(1 - \alpha^C_1 \right) \right| \leq  \delta_1$ and $\left| \eta^C_2 -  \left(1 - \alpha^C_2 \right) \right| \leq  \delta_2$.
\end{corollary}

\noindent
{\bf Remark:} Corollary~\ref{cor:groups} provides some useful insights towards achieving low infra-marginality. It quantitatively highlights the principle that---for mitigating group-wise infra-marginality---group-wise accuracy can be a better metric than overall accuracy. That is, in relevant settings, aiming for classifiers that maximize the minimum accuracy across groups (i.e., adopting a Rawlsian perspective on accuracy) can lead to more fair decisions than solving for, say, classifiers that enforce the same accuracy across groups or classifiers that maximize overall accuracy. In particular, Corollary~\ref{cor:groups} ensures that a max-min (Rawlsian) guarantee on accuracy translates into a max-min guarantee on infra-marginality, with additive errors at most $\delta_i$. \\

The following two propositions identify relevant settings wherein Theorem~\ref{theorem:main} can be applied.  The proofs of these propositions are direct and have been provided in the supplementary materials. The first proposition addresses data sets in which the probabilities $p^*_x$s (across data points $x \in \mathcal{X}$) are spread enough and do not sharply peak around a specific value. Formally, we say that a data set $\mathcal{X}$ is $\lambda$-Lipschitz if for any $ z \in [0, 1]$, the number of data points with $p^*_x \in \left[z, z + \frac{1}{|\mathcal{X}|} \right]$ is at most $\lambda$. Note that, for settings in which the cdf $F_{\mathcal{D}_{p^*}}$ of $\mathcal{D}_{p^*}$ is smooth, the maximum slope of the cdf corresponds to the Lipschitz constant of the data set. 

\begin{proposition}
If a data set $\mathcal{X}$ is $\lambda$-Lipschitz and the underlying fair-threshold $\tau = 0.5$, then the $\ell_1$ distance between the outcome probabilities $(p^*_x)_{x \in \mathcal{X}}$ and the single-threshold labels $( f^*_x )_{x \in \mathcal{X}}$ is at most $\lambda/4$. Here, we assume that the underlying feature distribution selects instances from $\mathcal{X}$ uniformly at random.
\end{proposition}

\begin{proof}
We partition $[0,1]$ into $N = |\mathcal{X}|$ subintervals, of length $1/|\mathcal{X}| = 1/N$ each. Specifically, the subintervals are $\left[ \frac{i-1}{N}, \frac{i}{N} \right]$, with integer $i \in \{1, 2, \ldots, N\}$. The Lipschitz condition ensures that the number of data points in each subinterval is at most $\lambda$. 

\begin{align*}
\| \mathcal{D}_{p^*} - \mathcal{D}_{f^*} \|_1 & \leq \frac{1}{N}\sum_{i = 1}^{N/2-1} \frac{i}{N} \lambda  \ + \ \frac{1}{N}\sum_{i = N/2}^N \left( 1 - \frac{i-1}{N} \right) \lambda \\
& = \frac{\lambda}{N^2} \left( \sum_{i = 1}^{N/2-1} i + \sum_{i = 1}^{N/2+1} i \right) = \frac{\lambda}{4} + o(1) 
\end{align*} 
\end{proof}

The next proposition observes that if, in $\mathcal{D}_{p^*}$, the probability mass is spread sufficiently far away from the threshold, then again the $\ell_1$ distance between the outcome probabilities and the single-threshold labels is appropriately bounded. The result shows that Theorem~\ref{theorem:main} is useful, in particular, when $\mathcal{D}_{p^*}$ is a bimodal distribution, with the two modes being close to zero and one, respectively. Formally, we will say that a distribution (supported on $[0,1]$) is \emph{$(\delta, q)$-spread}, with respect to the threshold $\tau=0.5$, iff, under $\mathcal{D}_{p^*}$, the probability mass in the interval $[0.5 - \delta, 0.5 + \delta]$ is at most $(1-q)$. Here, $0 \leq \delta \leq 1/2$ and $q \in (0, 1)$.  

\begin{proposition}
If, for a data set, the outcome probability distribution is $(\delta, q)$-spread and the underlying fairness threshold $\tau = 0.5$, then the $\ell_1$ distance between the outcome probabilities $(p^*_x)_{x \in \mathcal{X}}$ and the single-threshold labels $( f^*_x )_{x \in \mathcal{X}}$ is at most $1/2 - \delta q $. Here, we assume that the underlying feature distribution selects instances from $\mathcal{X}$ uniformly at random.
\end{proposition}
\begin{proof}
If $\mathcal{D}_{p^*}$ is $(\delta, q)$-spread, them the $\ell_1$ distance (between this distribution and $\mathcal{D}_{f^*}$) is maximized when $(1-q)$ fraction of the data points have $p^*$ value equal to $0.5$ and the rest of the data points (accounting for the remaining $q$ fraction) have $p^*$ value equal to $0.5 - \delta$ (or $0.5 + \delta$). Hence, the $\ell_1$ distance between $(p^*_x)_{x \in \mathcal{X}}$ and $( f^*_x )_{x \in \mathcal{X}}$ is upper bounded as follows: $ (1-q)  0.5 + q  ( 0.5 - \delta) = 1/2 - \delta q$. 
\end{proof}

\subsection{Optimal Classifiers Under Infra-marginality Constraints}
\label{section:opt-class}
The above subsection characterized infra-marginality under a fixed classifier threshold $\tau$. However, in practice, it is possible to choose $\tau$ to ensure high accuracy and also a minimum degree of infra-marginality.   
Therefore, we now present an efficient algorithm for finding classifiers with as high an accuracy as possible, under infra-marginality constraints. We address this optimization problem in a setup wherein the data set $\mathcal{X}$ is finite and outcome probabilities, $p^*$s, are known explicitly. In this setup (i.e., given $(p^*_x)_{x \in \mathcal{X}}$), prominent group-wise fairness notions map to non-convex constraints. Hence, maximizing accuracy subject to a group-wise fairness constraint typically requires relaxations (leading to approximate solutions) or heuristics; see, e.g.,~\cite{celis2018classification} and references therein. We will show that, by contrast, an upper bound on infra-marginality can be expressed as a linear constraint and, hence, we can efficiently find an optimal (with respect to accuracy) classifier that satisfies a specified infra-marginality bound. 

Let $(x, y^*_x)$ be a random sample from the feature distribution on $\mathcal{X} \times \{0, 1\}$, where $\mathcal{X}$ is the data set/feature space and the label $y^*_x \sim {\rm Bernoulli} (p^*_x)$, for each $x$. Recall that, given a universal threshold $\tau$ (which imposes the single-threshold fairness criterion), we define the label $f^*_x \coloneqq \mathbbm{1}\left\{ p^*_x > \tau \right\}$, for each $x$. Let $C: \mathcal{X} \mapsto \{0, 1\}$ be any classifier with accuracy $\alpha^C \coloneqq \E_{(x, y^*_x)} \left[ \mathbbm{1}\{ {C(x) = y^*_x} \} \right]$. Furthermore, the degree of infra-marginality of $C$ is defined as 
$\eta^C \coloneqq \E_{x} \left[ \left|C(x) - f^*_x \right| \right]$.

Given parameter $\eta$, we consider the problem of maximizing the accuracy, over all classifiers, subject to the constraint that the degree of infra-marginality is at most $\eta$. 

\begin{align*}
\text{maximize} \quad & \E_{(x, y^*_x)} \left[ \mathbbm{1}\{C(x) = y^*_x\} \right]  \quad \text{over $C: \mathcal{X} \rightarrow \{0, 1\}$} \\
\text{subject to} \quad &   \E_{x} \left[ \left|C(x) - f^*_x \right| \right] \leq \eta.
\end{align*}

Observe that, because $C: \mathcal{X} \mapsto \{0, 1\}$ and $y^*_x$ is binary valued, the objective function can be expressed as 
\begin{align*}
 \E_{(x, y^*_x)} \left[ \mathbbm{1} \{C(x) = y^*_x\} \right] & = \E_{(x, y^*_x)} \left[ C(x) y^*_x + (1 - C(x))(1-y^*_x) \right] \\
& = \E_x \left[ C(x) p^{*}_x + (1 - C(x))(1 - p^{*}_x) \right]\\
& = 1 - \E_x \left[ p^{*}_x \right] + \E_x \left[ {(2p^{*}_x - 1) C(x)} \right]
\end{align*}

Similarly, using the fact that $C: \mathcal{X} \mapsto \{0, 1\}$ and $f^*_x$ is binary valued, for infra-marginality we get 
\begin{align*}
\E_{x} \left[ \left|C(x) - f^*_x \right| \right] & = \E_x \left[ 1- C(x) f^*_x - (1 - C(x)) (1 - f^*_x) \right] \\
& = \E_x \left[ (1- 2 f^*_x) C(x) \right] + \E_x \left[ f^*_x \right] 
\end{align*}

Therefore, the above maximization of accuracy (or minimization of error rate) subject to the degree of infra-marginality being upper bounded by $\eta$, over classifiers $C: \mathcal{X} \mapsto \{0, 1\}$, can be equivalently rewritten as follows. Write $\overline{f} \coloneqq \E_x \left[ f^*_x \right]$ and note that $\overline{f}$, $\E_x [p^*_x ]$, and $\eta$ do not depend on the classifier $C$.
\begin{align*}
\text{minimize} \quad & \quad \E_x [ (1 - 2p^{*}_x ) C(x) ] \quad \text{over $C: \mathcal{X} \rightarrow \{0, 1\}$} \\
\text{subject to} \quad & \quad \E_x [ (2 f^*_x - 1) C(x) ] \geq \overline{f} - \eta.
\end{align*}

This is an integer linear program with (binary) decision variables $\{ C(x) \}_{x\in \mathcal{X}}$. Now consider its linear programming relaxation by letting $C: \mathcal{X} \mapsto [0, 1]$, and denote its optimum by $\alpha_{LP}^{*}$. Now, using Lagrange multipliers and strong duality, the optimum $\alpha_{LP}^{*}$ of the linear relaxation is given by
\begin{align*}
\alpha_{LP}^{*} & = \min_{C: \mathcal{X} \mapsto [0, 1]} \max_{\lambda \geq 0} \E_x [(1 - 2p^{*}_x) C(X) ]   -   \lambda \left( \E_x [ (2 f^*_x  -  1) C(x) ]   -  \overline{f}  +  \eta \right) \\
& = \max_{\lambda \geq 0} \ \min_{C: \mathcal{X} \mapsto [0, 1]}~ \E_x \left[ \left( (1 - 2p^{*}_x)  - \lambda (2 f^*_x - 1) \right) C(x) \right] + \lambda (\overline{f} - \eta) 
\end{align*}

For any fixed $\lambda$, the optimal classifier $C: \mathcal{X} \mapsto [0,1]$ is given by $C(x) = \mathbbm{1} \{ r_\lambda(x) < 0 \}$, where $r_\lambda(x) \coloneqq \left( (1 - 2p^{*}_x)  - \lambda (2 f^*_x - 1) \right)$. Therefore, 
\begin{align}
\alpha^*_{LP} & = \max_{\lambda \geq 0} \ \E_x \left[ r_\lambda(x) \ \mathbbm{1} \{ r_\lambda(x) < 0 \} \right] + \lambda (\overline{f} - \eta) \nonumber \\
& = \max_{\lambda \geq 0} \ \E_x \left[ \min \left\{ 0, r_\lambda(x) \right\} \right] + \lambda (\overline{f} - \eta) \label{eq:opt-lambda}
\end{align} 

With probabilities $(p^*_x)_{x \in \mathcal{X}}$ in hand, we can solve the dual of the linear programming relaxation to efficiently compute the optimal solution $\lambda^*$ of \eqref{eq:opt-lambda}, i.e., compute an optimal value of the Lagrange multiplier (dual variable). Given the optimal $\lambda^{*}$, we know that the optimal classifier $C^{*}: \mathcal{X} \rightarrow [0, 1]$ satisfies  
\begin{align*}
C^{*}(x) = \mathbbm{1}\left\{ r_{\lambda^*}(x) < 0 \right\} =  \mathbbm{1} \left\{ p^*_x + \lambda^*  f^*_x \geq \frac{1}{2} + \frac{\lambda^*}{2} \right\}\text{for all $x \in \mathcal{X}$}.
\end{align*}
In other words,
\begin{align*}
C^{*}(x) = \begin{cases} \mathbbm{1} \left\{ p^{*}_x \geq (1 + \lambda^{*})/2 \right\}, \quad \text{if $p^{*}_x \leq \tau$} \\ \mathbbm{1} \left\{ p^{*}_x \geq (1 - \lambda^{*})/2\right\}, \quad \text{if $p^{*}_x > \tau$} \end{cases}.
\end{align*}

Note that the optimal solution $C^{*}$ of the linear programming relaxation actually yields a binary classifier $C^{*}: \mathcal{X} \rightarrow \{0, 1\}$, which means that $C^{*}$ is also an optimal solution of the underlying integer linear program.\\

\noindent
{\bf Remark:} It is well-known that the Bayes classifier, giving the decisions $\mathbbm{1} \left\{ p^*_x \geq 1/2 \right\}$, achieves maximum accuracy over all classifiers. Our derivation shows that, even with infra-marginality constraints, optimal classifiers continue to be single-threshold classifiers. Also, it is relevant to note that the above-mentioned method can be used to efficiently solve the problem of minimizing infra-marginality subject to a (specified) lower bound on accuracy. 

\section{Empirical Evaluation}
\label{sec:exp}
In this section, our primary aim is to provide empirical evidence to the developed theoretical guarantees and answer the following questions:

\begin{itemize}
\item Does high accuracy (low classification error) lead to low infra-marginality for a classifier?
\item How does the relationship between classification-error and infra-marginality change when we consider them separately for each protected group?
\item What is the nature of trade-off between optimizing a classifier for group fairness measures versus infra-marginality? 
\end{itemize}

\subsection{Experimental Setup}
\label{sec:subexp}
We answer these questions by measuring accuracy, infra-marginality and group fairness metrics for machine learning classifiers under a wide range of datasets. Since estimation of infra-marginality depends on knowledge of the true outcome distribution, we first present results on simulated datasets where we control the data generation process. The distributions are chosen to provide a thorough understanding of the relationship between accuracy and degree of \textit{infra-marginality}. We then present results on two real-world datasets that have been used in prior empirical work on algorithmic fairness: \texttt{Adult Income}~\cite{url:adultIncome} (for predicting annual income) and \texttt{MEPS}~\cite{url:meps} (for predicting utilization using medical expenditure data). These datasets satisfy the two properties that are required for empirical application of infra-marginality. First, the outcomes measured are numerical quantities that are less likely to be subjectively biased and second, they can be assumed to be a representative sample of the underlying population. The \texttt{Adult-Income} dataset contains  a sample of adults in the United States based on the Census data from 1994 and the \texttt{MEPS} dataset is derived from a nationally representative survey sample of people's healthcare expenditure in the US.        

\noindent \textbf{Measuring infra-marginality. }On the simulated datasets, we estimate infra-marginality using its definition in Equation~\ref{eq:infra}. On real-world datasets, we do not have true outcome probability and thus use an approximation. We assume that the estimated outcome probability ($\hat{p}_x$) using a classifier can be considered as a proxy for the true outcome probability. Essentially, this assumption implies that all the relevant variables for estimating the outcome are available in the dataset (and that we have an optimal learning algorithm to estimate the outcome probability).
\begin{align}
    \eta^C \coloneqq \E_{x} \left[ \left|C(x) - \mathbbm{1}{[\hat{p}_x>\tau]} \right| \right]
\end{align}
To assess the sensitivity of results to settings where the true outcome probability $p^*_x$ is different from the one estimated by the learnt classifier, we develop a method: we progressively remove features from a given dataset, train a classifier on the reduced dataset, and compare infra-marginality on the outcome probability estimated from this reduced dataset to the ``true'' infra-marginality on the outcome probability estimated from the full dataset (thus assuming it as the true outcome probability). That is, we use the following measure for infra-marginality. 
\begin{align}
    \eta^C \coloneqq \E_{x'} \left[ \left|C(x') - \mathbbm{1}{[\hat{p}_{\mathit{all\_features}}>\tau]} \right| \right] \textit{s.t. }X' \subset X 
\end{align}

While we demonstrate the theoretical results using the above assumption on the real-world datasets, in practice we recommend a combination of domain knowledge and active data collection to estimate true outcome probabilities and correspondingly, the true infra-marginality. We discuss these possibilities in Section~\ref{sec:conclude}.

\noindent \textbf{Measuring group-fairness. }
For each classifier on the simulated and real-world datasets, we compare infra-marginality to prominent group-fairness metrics. We do so using the \textit{meta-fair framework} proposed by Celis et al.~\cite{celis2018classification}. In this framework, there is a trade-off parameter which helps to balance between maximizing accuracy and achieving group-fairness. Higher the value of the parameter, the more is the focus on achieving group-fairness. When this parameter is $0$, the classifier maximizes accuracy with no group-fair constraints. For our evaluation, we consider two group-fair notions:
\begin{itemize}
    \item \emph{Statistical Parity (SP) or Disparate Impact (DI):} The aim is to achieve low value for the expression $\left(1-\min\left\{\frac{SR_{1}}{SR_{0}},\frac{SR_{0}}{SR_{1}}\right\}\right)$, where $SR_z$ denotes the \textit{selection-rate} for group $z$, which is the fraction of individuals who received favorable outcome by the classifier, within group $z$.
    \item \emph{Equal False Discovery Rates:} The aim is to achieve low value for $\left(1-\min\left\{\frac{FDR_{1}}{FDR_{0}},\frac{FDR_{0}}{FDR_{1}}\right\}\right)$, where $\mathit{FDR}_z$ denotes the false discovery rate for group $z$, which is the fraction of individuals incorrectly classified, among the group $z$ who received favorable outcomes by the classifier.
\end{itemize}
We use the implementation of the meta-fair algorithm, provided in the Python package AI Fairness 360~\cite{aif360-oct-2018}. 
Finally, we report the classification error rate, i.e., (1 - accuracy), degree of infra-marginality, and the value of the group (un)fairness metric.

\subsection{Simulation-Based Datasets}
For simplicity, we consider datasets with a single sensitive attribute (e.g., race or gender) and two additional attributes (e.g., age, income, etc.) that denote relevant features for an individual. We assume that the sensitive attribute is binary and additional attributes are continuous. We also assume that the outcome is binary and depends on the attributes of an individual. Given a classifier that predicts the outcome based on these features, our central goal is to compare its misclassification rate, infra-marginality and group-fairness.    

Specifically, we assume that the sensitive attribute $Z \in \{0,1\}$, two additional attributes $U$, $V$ are real-valued,  and the outcome is binary $Y\in \{0,1\}$. Hence, the input space is denoted as $\mathbb{X}=\{0,1\}\times \mathbb{R} \times \mathbb{R}$. We create datasets using a generative model where the attributes of an individual are simulated based on their sensitive attribute and outcome. The overall population distribution is generated as 
$\mathcal{P}(U,V,Z,{Y})=\mathcal{P}(U,V|Z,{Y})\cdot \mathcal{P}(Z|{Y})\cdot \mathcal{P}({Y})$. We further consider that $U$ and $V$ are conditionally independent given $Z$ and ${Y}$, i.e., $\mathcal{P}(U,V|Z,{Y})=\mathcal{P}(U|Z,{Y})\cdot\mathcal{P}(V|Z,{Y})$, and the distributions are considered to be Gaussian. Within this framework, we generate five types of datasets, each with $10000$ instances and equal label distribution $\mathcal{P}(Y=1)=\mathcal{P}(Y=0)=0.5$. For each instance $x \in \mathbb{X}$, we obtain $p^*_x=\mathcal{P}(Y=1|x)$ by applying Bayes' rule. Then we obtain $y^*_x\sim \mathtt{Bernoulli}(p^*_x)$ and $f^*_x = \mathtt{Indicator}(p^*_x > 0.5)$. 
Based on this process, we generate datasets with various outcome probability ($p^*$)  distributions, as shown in Figure~\ref{fig:p-distr}. 

\begin{figure}[ht]
    \centering     
\begin{subfigure}[b]{0.45\textwidth}
    \includegraphics[width=\textwidth]{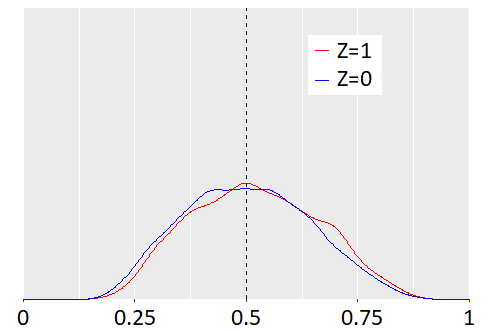}  
    \caption{S1}
    \label{fig:risk_close}
\end{subfigure}     ~ 
\begin{subfigure}[b]{0.45\textwidth} 
    \includegraphics[width=\textwidth]{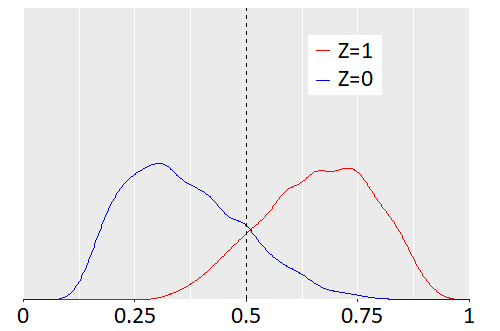} 
    \caption{S2}    
    \label{fig:risk_twoPeaks_smooth} 
\end{subfigure} 
\begin{subfigure}[b]{0.45\textwidth} 
    \includegraphics[width=\textwidth]{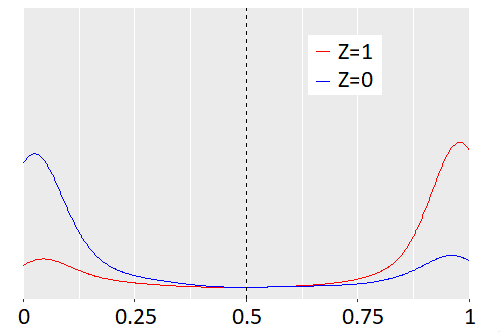} 
   
    \caption{S3}    
    \label{fig:risk_twoPeaks_steep} 
\end{subfigure}  ~
\begin{subfigure}[b]{0.45\textwidth} 
    \includegraphics[width=\textwidth]{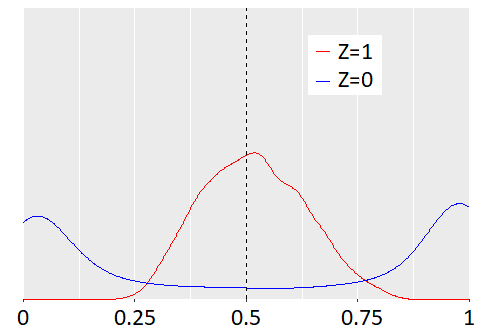} 
    \caption{S4}    
    \label{fig:risk_diff} 
\end{subfigure}   
\caption{True outcome probability distribution for the datasets \textbf{S1, S2, S3, S4}. The x-axis denotes the $p^*$ and y-axis denotes the density $\mathcal{D}_{p^*}$. The blue and red lines represent density curves for the group $Z=0$ and $Z=1$, respectively.}
\label{fig:p-distr} 
\end{figure}

\begin{itemize}[itemsep=0pt,leftmargin=*,topsep=2pt]
    \item \textbf{Dataset S1}: The risk distribution of this dataset is shown in Figure~\ref{fig:risk_close}. For any individual, irrespective of their group membership, the outcome probability is drawn from a density distribution $\mathcal{D}_{p^*}$ which is concentrated near the threshold of $0.5$. Thus, it is difficult to obtain highly accurate classifiers for this kind of outcome distribution. Also, the $\ell_1$ distance between $\mathcal{D}_{p^*}$ and $\mathcal{D}_{f^*}$ is $\delta=0.38$ overall as well as per group.
    \item \textbf{Dataset S2}: We then construct a dataset where the outcome distributions for the two sensitive groups are separated from each other. One group (Z=0) has an outcome probabiliity mode below 0.5, and the other group (Z=1) has a mode on outcome probability greater than 0.5.  The $\mathcal{D}_{p^*}$ distribution of this dataset is described in Figure~\ref{fig:risk_twoPeaks_smooth}. Compared to $S1$, we expect it to be easier for a classifier to achieve high accuracy, but achieving group statistical parity (Disparate Impact) is harder. Here, the two groups have different densities over the outcome probabilities, one primarily concentrated around $0.25$ whereas other is concentrated around $0.75$. Thus, in this dataset, imposing a group-fairness constraint such as disparate impact would necessarily cause more classification error and more degree of infra-marginality.  The $\delta$ value for this dataset is $0.33$ for both the groups.
    \item \textbf{Dataset S3}: This dataset corresponds to an extreme case of $\textbf{S2}$ where the outcome probability for one of the sensitive groups is concentrated near 0 whereas that for the other sensitive group is near 1, as shown in Figure~\ref{fig:risk_twoPeaks_steep}. Also the $\delta$ value is only $0.10$. Achieving high accuracy and low infra-marginality should be extremely easy for a threshold-based classifier, but satisfying group-wise parity may be equally hard. 
    \item \textbf{Dataset S4}: Here we generate data such that the density over outcome probabilities for group $Z$=$1$ is spread sufficiently away from the threshold $0.5$, whereas the density over outcome probabilities for group $Z$=$0$ is concentrated around $0.5$, shown in Figure~\ref{fig:risk_diff}. The $\delta$ values for two groups are also different, $0.41$ and $0.08$ for group $1$ and $0$, respectively. Thus, we expect to see a difference in accuracy by group: a classifier should be able to distinguish individuals from group $Z=0$ easily but find it hard to separate people with $1$ or $0$ outcomes within group $Z=1$. In contrast, any classifier with a single threshold is also expected to satisfy statistical parity (SP).
    \item \textbf{Dataset S5}: In this dataset, the feature $U$ is drawn from the same Gaussian distribution, irrespective of $Y$ or $Z$ values. However, feature $V$ is drawn such that the values clearly separate individuals with different $Y$ values, among group $Z=0$. For $Z=1$ group, it is difficult to accurately classify, as shown in Figure~\ref{fig:apart}.
    
\begin{figure}[ht]
\centering
\includegraphics[width=0.8\columnwidth]{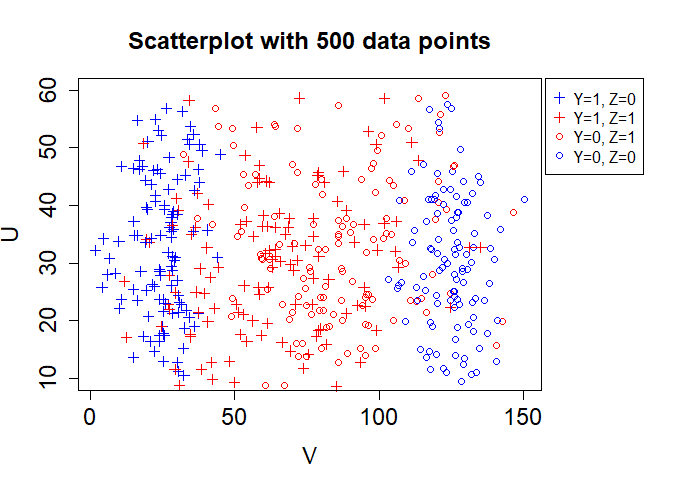}
\caption{Scatterplot for a subset of datapoints from dataset \textbf{S}$\mathbf{5}$.}
\label{fig:apart}
\end{figure}
\end{itemize}

For dataset S5, however, the optimal classifier's decision boundary would be $V=75$, whereas, a classifier ensuring equal misclassification rates for both subgroups would necessarily choose the decision boundary as a horizontal line, say $U=35$. In this case, a large fraction of individuals would receive different decision than that of the optimal classifier, hence would suffer from higher infra-marginality. \\

\noindent \textbf{High accuracy, Low infra-marginality. }
Figure~\ref{fig:delta} illustrates that the degree of infra-marginality is well within the theoretical bound of classification\_error$\pm\ \delta$, as established in Theorem~\ref{theorem:main}. 
In particular, Figure~\ref{fig:risk_close_delta} shows that the classification error is $30\%$ when a classifier is trained to maximize accuracy for dataset S1 and the degree of infra-marginality is bounded by the error $\pm\ 0.38$ ($\delta=0.38$ overall, as well as groupwise). Interestingly, for dataset S2, the degree of infra-marginality is very close to the accuracy even when the theoretically established bound is high, as shown in Figure~\ref{fig:risk_twoPeaks_smooth_delta}. The $\delta$ for the dataset S3 is low ($0.10$), and hence, the infra-marginality and error rates are very close to the each other (the y-axis of Figure~\ref{fig:risk_twoPeaks_steep_delta} is limited to $0.5$ for clarity). The bound of infra-marginality continues to hold in Figures~\ref{fig:diff_delta} and \ref{fig:apart_delta}.\\

\begin{figure}[!ht]
    \centering     
\begin{subfigure}[b]{0.4\textwidth}
    \includegraphics[width=\textwidth]{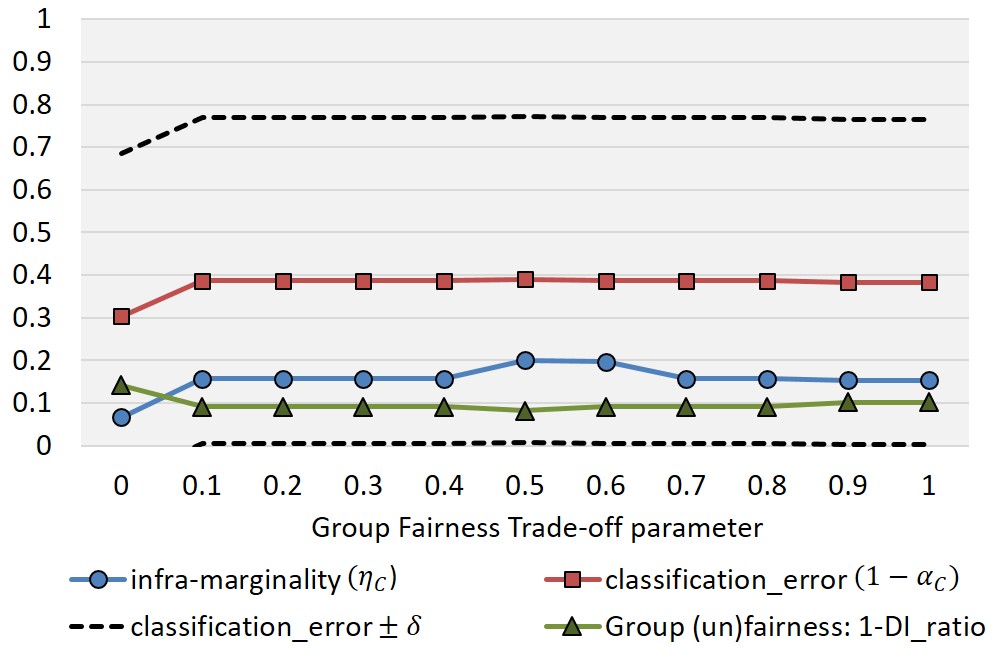}  
    \caption{S1}
    \label{fig:risk_close_delta}
\end{subfigure}     ~ 
\begin{subfigure}[b]{0.4\textwidth} 
    \includegraphics[width=\textwidth]{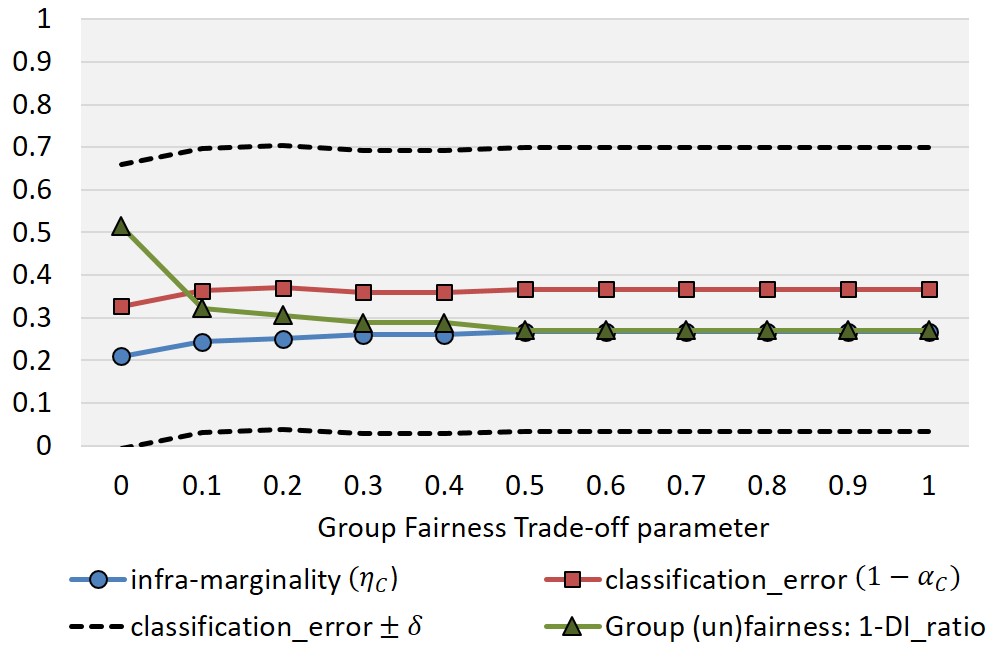} 
    \caption{S2}    
    \label{fig:risk_twoPeaks_smooth_delta} 
\end{subfigure} 
\begin{subfigure}[b]{0.4\textwidth} 
    \includegraphics[width=\textwidth]{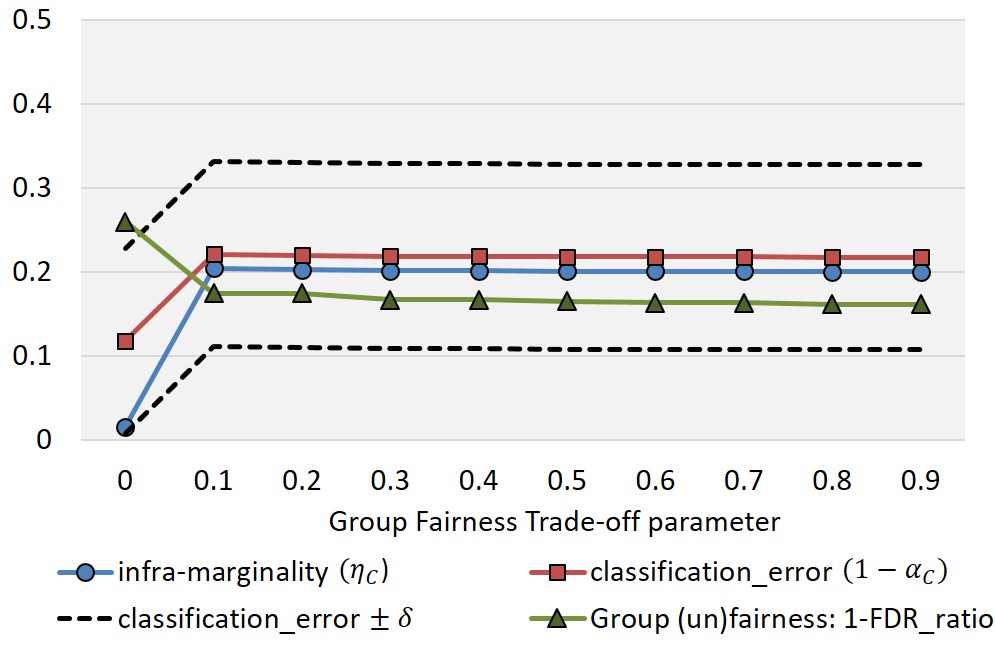}  
    \caption{S3}    
    \label{fig:risk_twoPeaks_steep_delta} 
\end{subfigure}  ~
\begin{subfigure}[b]{0.4\textwidth} 
        \includegraphics[width=\textwidth]{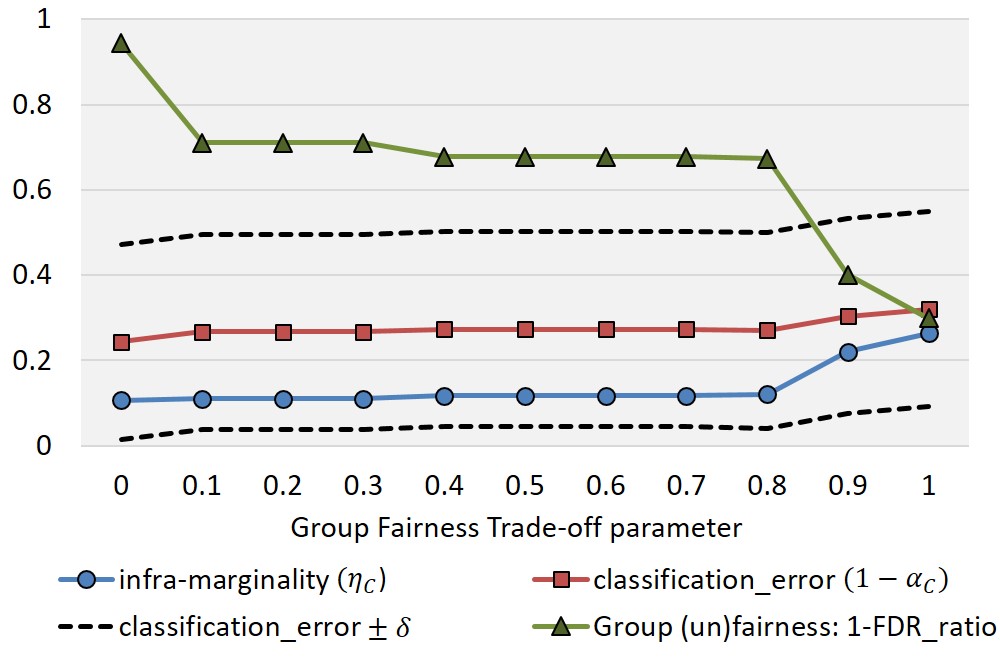} 
    \caption{S4}    
    \label{fig:diff_delta} 
\end{subfigure}   
\begin{subfigure}[b]{0.4\textwidth} 
        \includegraphics[width=\textwidth]{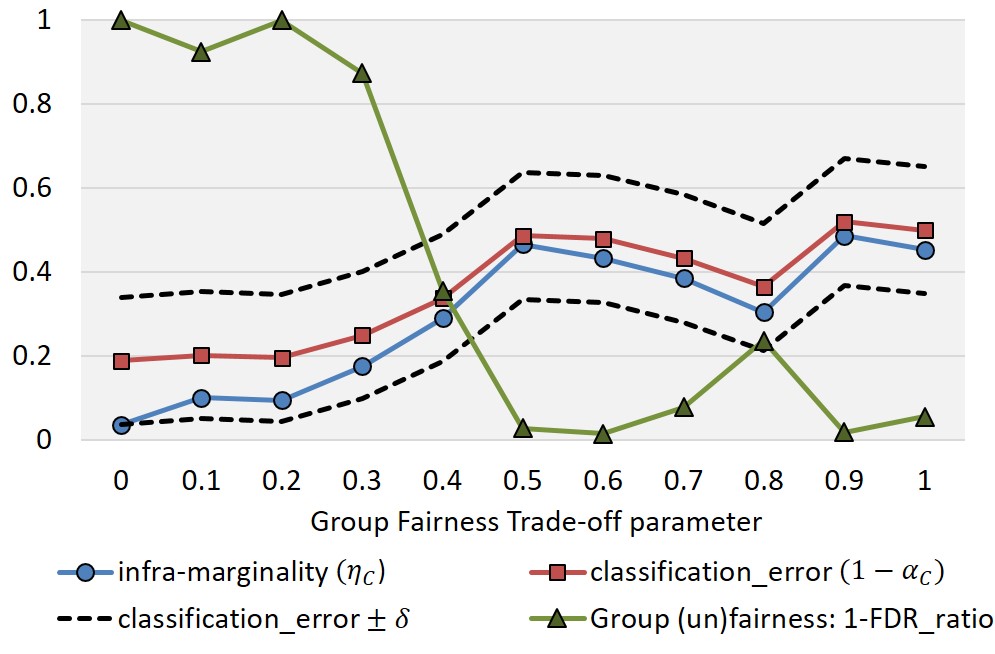} 
    \caption{S5}    
    \label{fig:apart_delta} 
\end{subfigure} 
\caption{Comparing the classification error, infra-marginality and group (un)fairness with increasing values of the group-fairness trade-off parameter in the meta-fair algorithm.}
\label{fig:delta} 
\end{figure}

\noindent \textbf{Tradeoff between infra-marginality and group fairness. }
Figures~\ref{fig:risk_close_delta} and \ref{fig:risk_twoPeaks_smooth_delta} show that, by invoking DI-fair constraint, the classification error and degree infra-marginality increases by $8-10\%$ more. Even for dataset S3, DI-fairness increases the classification error, and the corresponding graph obtained looks similar to that of the previous one (Figure~\ref{fig:risk_twoPeaks_smooth_delta}). 

For S3, we additionally evaluate the effect of using another fairness notion called False Discovery rate (FDR). Figure~\ref{fig:risk_twoPeaks_steep_delta}  shows that ensuring FDR-fairness hurts the accuracy by about $11\%$ and infra-marginality increases by about $19\%$.  
For S4, the FDR-fairness constraint not only causes increase classification error but also increases the false discovery rates for both subgroups (as shown in Figure~\ref{fig:twoPeaks_steep_groupwise}), which is an extremely undesirable consequence of ensuring group-fairness. 
We observe a similar trend for dataset S5 in Figure~\ref{fig:apart_delta}. Also, we observe that the ratio of false discovery rates do not uniformly improve on increasing the group-fairness parameter.\\

\begin{figure}[ht]
\centering
\includegraphics[width=0.7\textwidth]{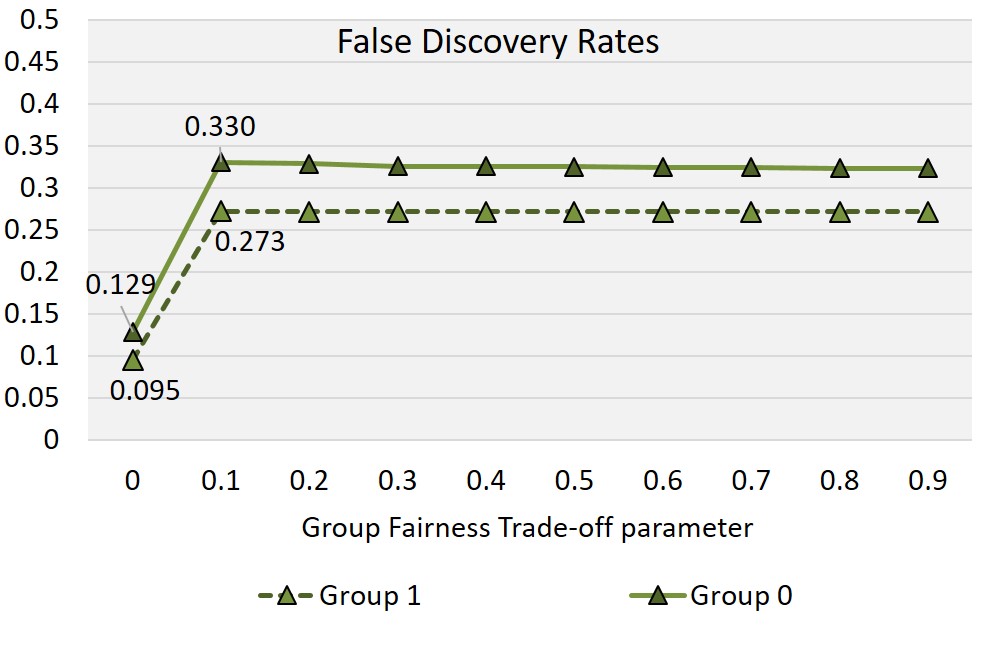}
\caption{FDRs for two groups increase while ensuring FDR-fairness, for S4.}
\label{fig:twoPeaks_steep_groupwise}
\end{figure}

\noindent \textbf{Group-wise accuracy and infra-marginality. } We show results for one of the datasets S4 by computing the metrics separately for each sensitive group; results on others are similar. 
In dataset S4, the $\delta$ values for group $0$ and $1$ are $0.08$ and $0.41$, respectively. Thus, the infra-marginality and error lines in Figure~\ref{fig:diff_delta_groupwise} (obtained using FDR-fairness constraint), almost coincides for group $0$, while these lines are quite apart for group $1$. Moreover, Figure~\ref{fig:diff_delta} shows that, when the trade-off parameter is $0.8$ or higher, the group-fairness improves, at a significant cost of infra-marginality and accuracy.\\

\begin{figure}[!h]
\centering
\begin{minipage}{\textwidth}
\centering
\includegraphics[width=0.48\textwidth]{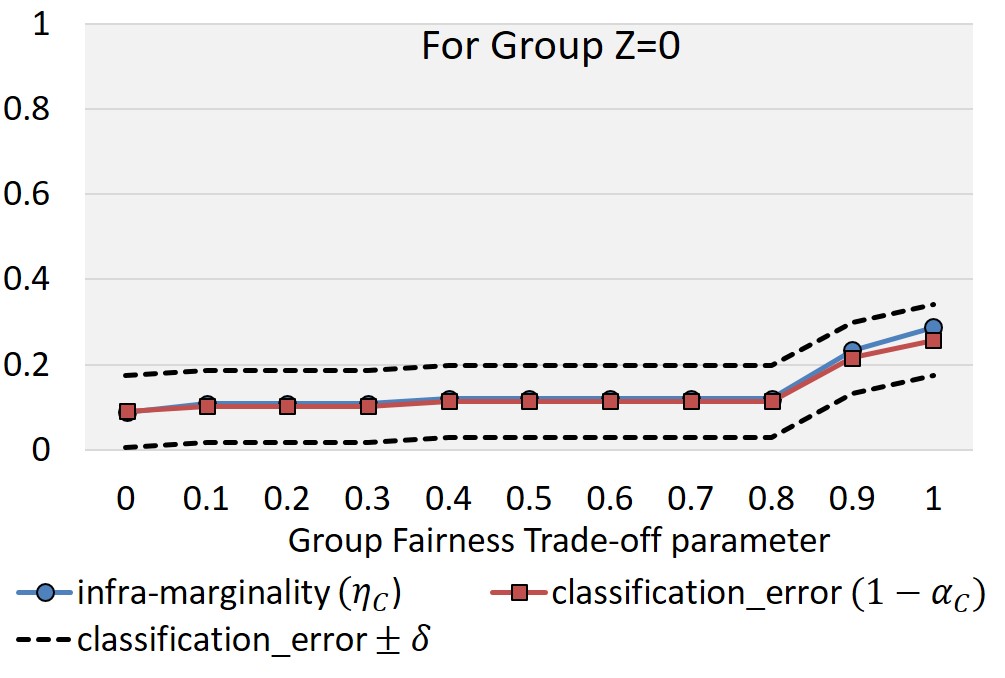}
\hfill
\includegraphics[width=0.48\textwidth]{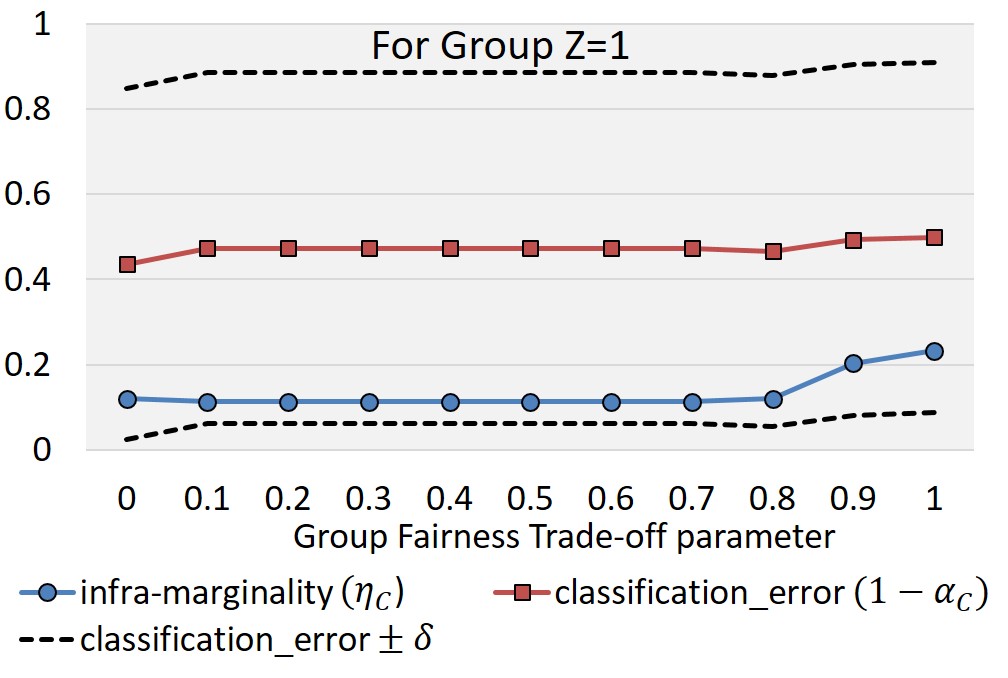}
\caption{Comparison of classification error and infra-marginality for dataset S$4$ with respect to two groups $Z=0$ and $Z=1$.\label{fig:diff_delta_groupwise}}
\end{minipage}

\end{figure}

\noindent \textbf{Summarizing the results on synthetically generated datasets}:
The observations over all the simulations summarize that the infra-marginality biases are within $\delta$ of the classifier's error rates. In fact, in almost all executions, when accuracy is the highest, the infra-marginality values are lower than $0.1$ ($10\%$). So, addressing the first question, we find that low classification error leads to low infra-marginality, even when the $\delta$ values are large. From our group-wise results on S4, we find a corresponding result:  classifiers are highly accurate for $Z=0$ and thus exhibit very low degree of infra-marginality, while accuracy for $Z=1$ group is around $60\%$ and thus this group suffers higher infra-marginality. Finally, we see that increasing group-fairness leads to worse accuracy, and in turn, higher degree of infra-marginality in all the datasets. 

\subsection{Case-Study}
We now describe our  results on the
\texttt{Adult} \texttt{Income} and \texttt{Medical} \texttt{Expense} datasets. 
\begin{itemize}[itemsep=0pt,leftmargin=*,topsep=2pt]
    \item \underline{\textit{Medical Expenditure Panel Survey dataset} (MEPS)}~\cite{url:meps}: This data consists of surveys of families and individuals collecting data on health services used, costs and frequency of services, demographics, etc., of the respondents. The classification task is to predict whether a person would have `high' utilization (defined as UTILIZATION $>= 10$, which is roughly the average utilization for the considered population). The feature `UTILIZATION', was created to measure the total number of trips requiring some sort of medical care, by summing up the following features: OBTOTV15 (the number of office based visits), OPTOTV15 (the number of outpatient visits), ERTOT15 (the number of ER visits), IPNGTD15 (the number of inpatient nights), and HHTOTD16, the number of home health visits. High utilization ($Y=1$) respondents constituted around $17\%$ of the dataset. The sensitive attribute, `RACE' is constructed as follows: `Whites` (Z=0) is defined by the features RACEV2X = 1 (White) and HISPANX = 2 (non Hispanic); everyone else are tagged `Non-Whites' (Z=1).
    \item \underline{\textit{Adult Income dataset}}~\cite{url:adultIncome}: This dataset contains complete information of about $30162$ individuals. It is obtained from the UCI Machine Learning Repository \cite{Dua:2017}. This dataset has been extensively used in supervised prediction of annual salary of individuals (whether or not an individual earns more than $\$ 50$K per year). High salary is used the outcome label $Y=1$. For our experiment, we consider features such as \texttt{age-groups}, \texttt{education-levels}, \texttt{race}, and \texttt{sex}. The column \texttt{income} represents $Y$ labels, which contains $1$ for individuals with $>\$50$K annual salary and $0$ otherwise. The \texttt{sex} attribute is considered to be a sensitive with $Z=1$ denoting ``Female'' individuals and $Z=0$ denoting ``Male''. \\
\end{itemize}

\noindent\textbf{Observations. } We follow a similar analysis to the simulated datasets by constructing classifiers with different group-fairness parameter and measuring infra-marginality. As described in Section~\ref{sec:subexp}, we first assume that $p^*$ can be approximated by the learnt outcome probabilities from the classifier. 
For MEPS, we observe a steep increase in infra-marginality and error rate when DI-fairness constraint is invoked (as shown in Figure~\ref{fig:meps_delta}). This result holds group-wise: when dividing the data by the sensitive group and considering them separately,  Figure~\ref{fig:meps_delta_groupwise} shows that adding the fairness constraint causes nearly $\mathbf{70\%}$ of the decisions to change compared to the classifier that maximizes accuracy within both groups, thus leading to high infra-marginality within both groups. \\

\begin{figure}[ht]
\centering
\includegraphics[width=0.6\columnwidth]{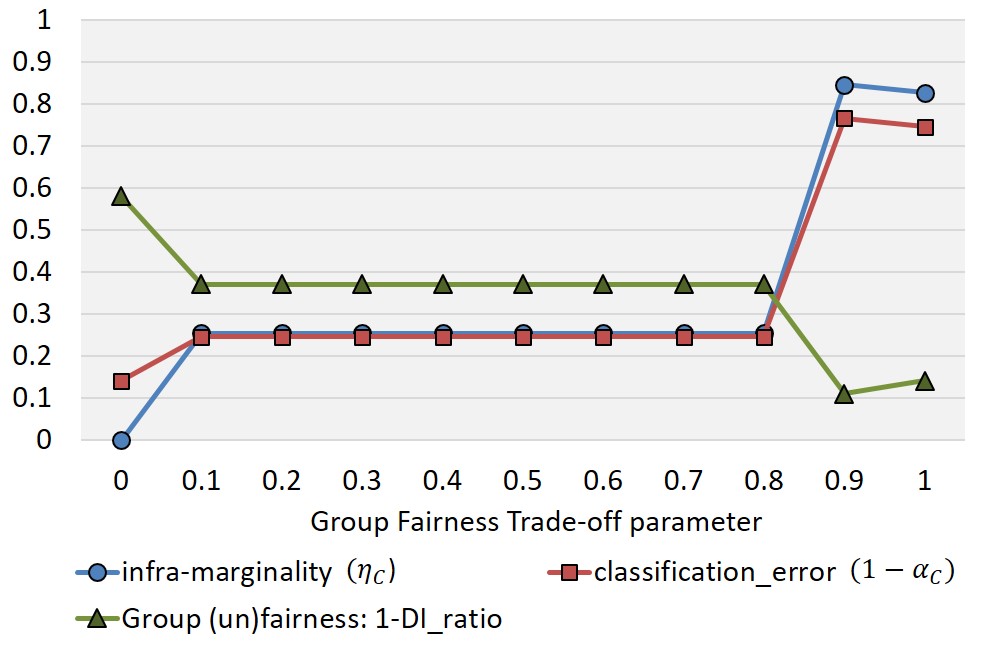}
\caption{Comparing classification error, infra-marginality and group (un)fairness with increasing values for the group-fairness trade-off parameter in the meta-fair algorithm. These classifiers are trained using the \texttt{MEPS} dataset.}
\label{fig:meps_delta}
\end{figure}

\begin{figure}[!ht]
\centering
\begin{subfigure}[t]{0.48\columnwidth}
\centering
\includegraphics[width=\textwidth]{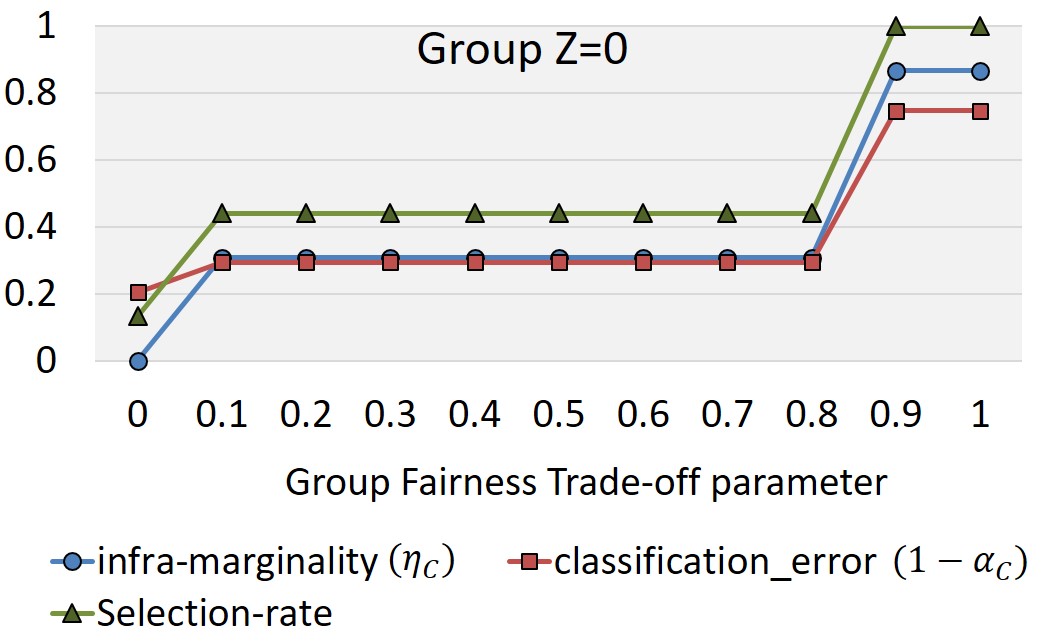}
\end{subfigure}
\hfill
\begin{subfigure}[t]{0.48\columnwidth}
\centering
\includegraphics[width=\textwidth]{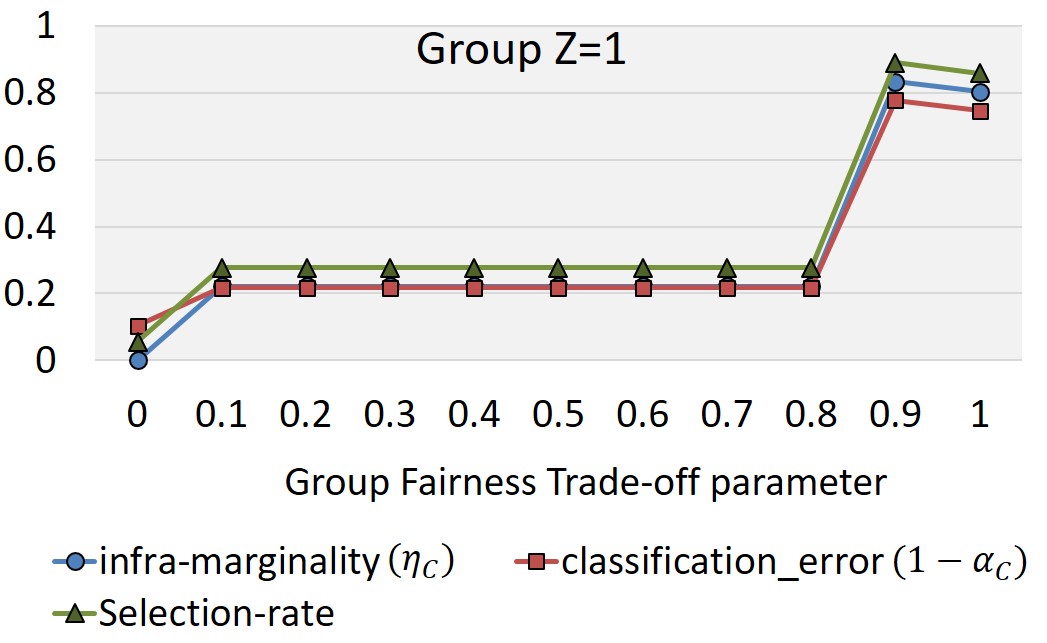}
\end{subfigure}
\caption{Groupwise comparison of classification error, infra-marginality and selection rate for \texttt{MEPS}  dataset.  \label{fig:meps_delta_groupwise}}
\end{figure}

Similarly for Adult Income dataset, we observe an increase in the degree of infra-marginality and error rate when FDR-fairness constraint is invoked (as shown in Figure~\ref{fig:adult_delta}). When looking at these metrics group-wise (Figure~\ref{fig:adult_delta_groupwise}), we see that adding the fairness constraint causes nearly $\mathbf{50\%}$ of the decisions to change (infra-marginality), among both the groups. Moreover, it also causes increase in the false discovery rates for both the groups; in particular, FDR is  more than $0.8$ for group $Z=0$.

\begin{figure}[!h]
\centering
\includegraphics[width=0.6\columnwidth]{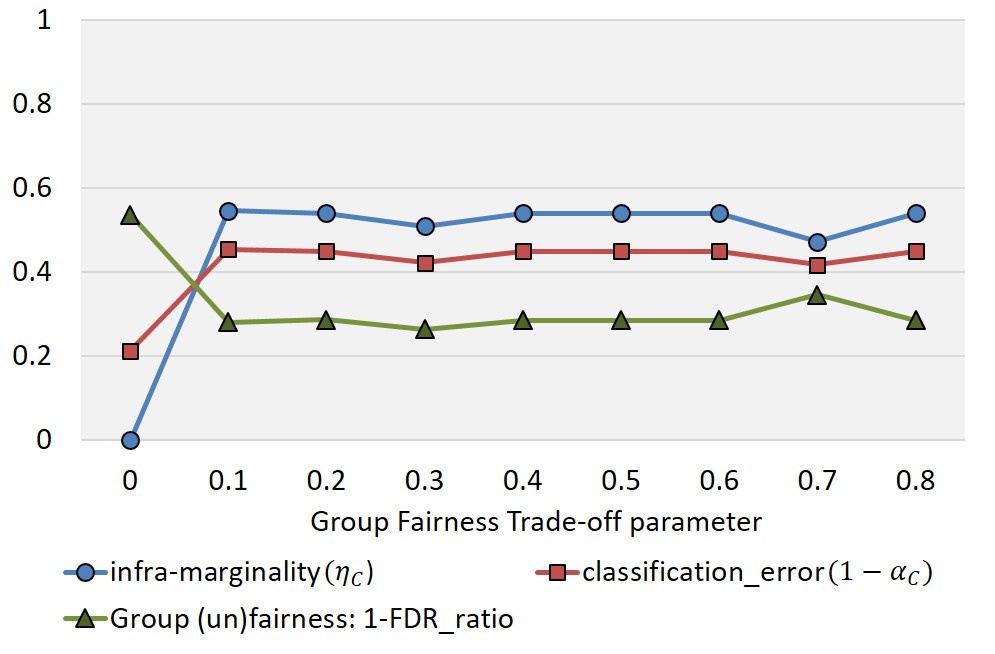}
\caption{Comparing classification error, infra-marginality and group (un)fairness with increasing values for the group-fairness trade-off parameter in the meta-fair algorithm. These classifiers are trained using the \texttt{Adult} \texttt{Income} dataset.}
\label{fig:adult_delta}
\end{figure}

\begin{figure}[!h]
\centering
\begin{subfigure}[t]{0.48\columnwidth}
\centering
\includegraphics[width=\textwidth]{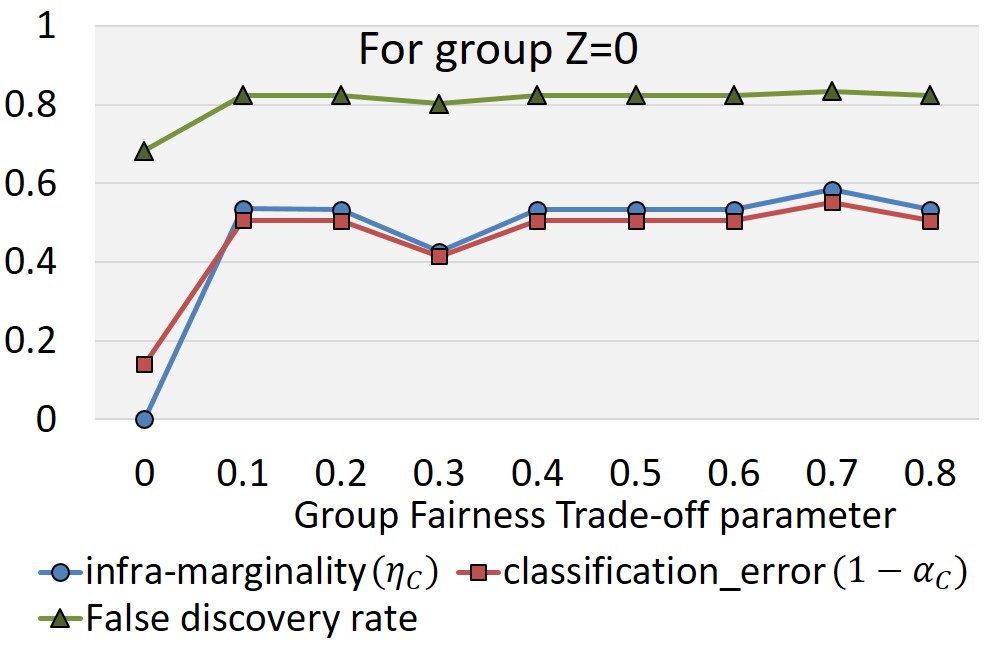}
\end{subfigure}
\hfill
\begin{subfigure}[t]{0.48\columnwidth}
\centering
\includegraphics[width=\textwidth]{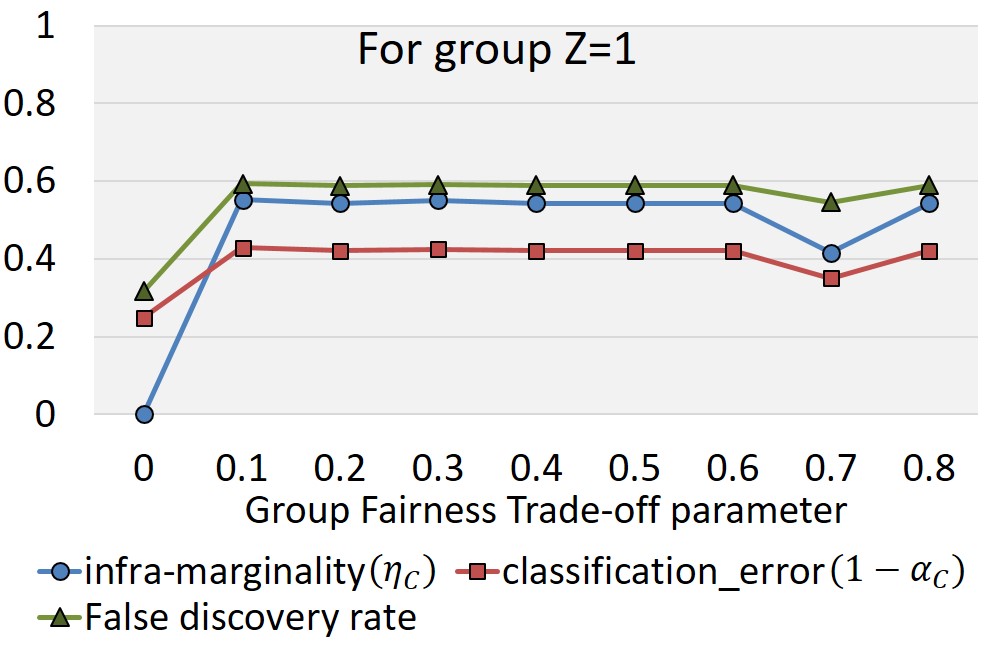}
\end{subfigure}
\caption{Groupwise comparison of classification error, infra-marginality and false discovery rate (FDR) for \texttt{Adult} \texttt{Income} dataset.  \label{fig:adult_delta_groupwise}}
\end{figure}

\begin{figure}[!h]
\centering
\begin{subfigure}[t]{0.48\columnwidth}
\centering
\includegraphics[width=\textwidth]{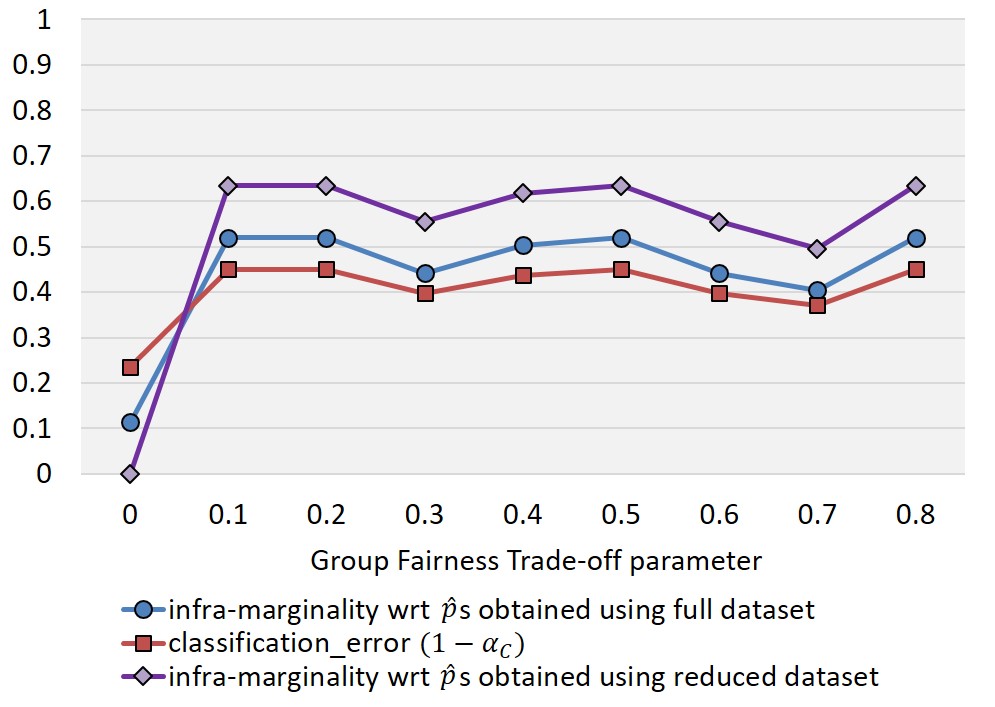}
\caption{Reduced dataset after removing \texttt{education-levels}}
\end{subfigure}
\hfill
\begin{subfigure}[t]{0.48\columnwidth}
\centering
\includegraphics[width=\textwidth]{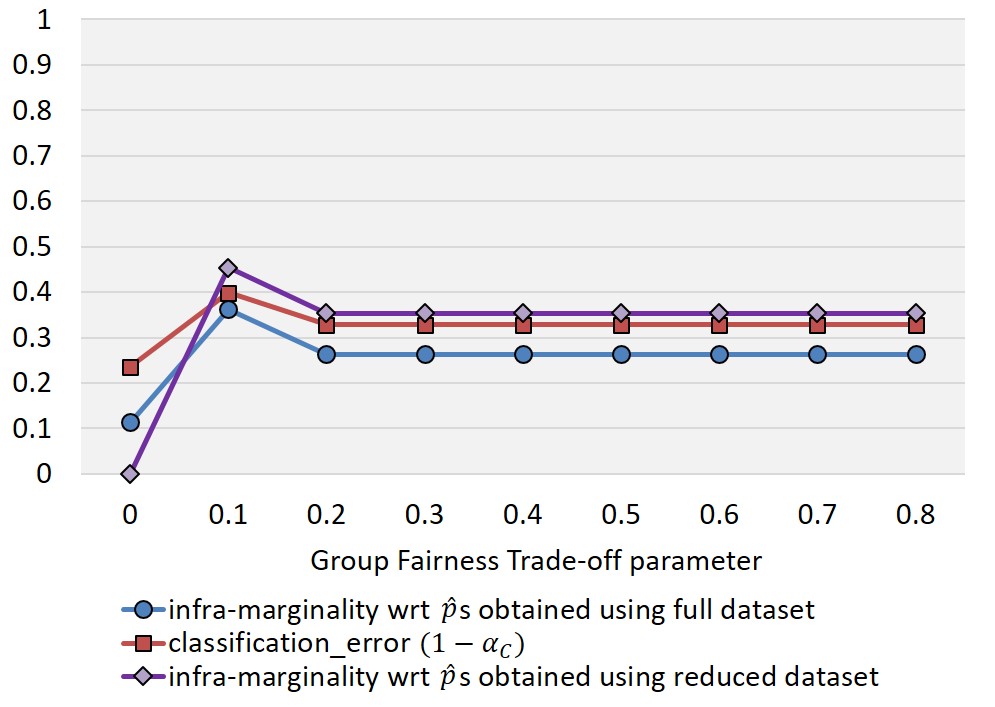}
\caption{Reduced dataset after removing \texttt{education-levels} and \texttt{race}}
\end{subfigure}
\caption{Comparison on \texttt{Adult Income} dataset, after removing (one or more) features.  \label{fig:adult_subset_delta}}
\end{figure}

So far, however, we assumed that $\hat{p}_x$ is a proxy for the true outcome probability. We now repeat the experiments on Adult Income data after leaving out an important feature \texttt{education-levels} and then leaving out \texttt{race} also. For these reduced datasets, we assume that the true outcome probability can be derived from the full dataset with all features ($p^*_x \approx \hat{p}_{allfeatures}$ ), but the classifier only has access to the reduced features. We observe that, even when all the features are not available, low  infra-marginality with respect to full dataset is associated with low classification error (Figure~\ref{fig:adult_subset_delta}). We also observe that infra-marginality with respect to the $\hat{p}$s obtained using reduced dataset is correlated to the infra-marginality with respect to the $\hat{p}$s obtained using the full dataset. This observation help us claim that, in real-world datasets (where we may not have full information of all the features), we can quantify infra-marginality using $\hat{p}s$ obtained from accuracy-maximizing classifier.\\

\noindent \textbf{Summarizing the results on real-world datasets}:
We observe that, in both \texttt{Adult Income} and \texttt{MEPS} datasets, low classification error led to low infra-marginality. Further, we see that increasing the extent of group-fairness may lead to worse accuracy and, in turn, higher degree of infra-marginality. We observe similar trend in the group-wise results for \texttt{Adult Income} dataset. Finally, we investigate whether unavailability of one or more features would have an impact in the above observations. We find that, even with less features, the infra-marginality remains lowest when the classification error is the lowest.   

\section{Concluding Discussion}\label{sec:conclude}
We provided a metric for infra-marginality due to a machine learning classifier and characterized its relationship with accuracy and group-based fairness metrics. 
In cases where unbiased estimation of the true probability of outcome, $p^*$, is possible, our theoretical results indicate the value of considering infra-marginality in addition to prominent group fairness metrics. Moreover, we showed that optimizing for infra-marginality results in a linear constraint that can be efficiently solved, in contrast to the non-convex constraints that typical group fairness metrics impose.

However, measuring infra-marginality requires estimating $p^*$ for any decision-making scenario, which is not usually easy to obtain. In some settings, such as the search for contraband by stopping cars on the highway~\cite{simoiu2017problem}, we do obtain a proxy for $p^*$, by making assumptions on the police officers' decision-making. This is also possible in other law enforcement contexts, such as searching for prohibited items at airport security, wherein one can argue that security officials stop a person only if they detect a suspect object through the X-ray and, thus, the logged data of baggage searches can be assumed to be an unbiased estimate of $p^*$. Further, airport security might decide to perform random searches, which can serve as ``gold-standard'' unbiased estimators of $p^*$.

When such data is not logged, we suggest actively changing the current decision-making process to introduce some unbiased data that can be used for estimating $p^*$. This can be done by adding probabilistic decisions, such as randomly deciding to search a person, awarding a loan to a small fraction of applicants, and so on. Decisions need not be fully random: recent work on multi-armed and contextual bandits~\cite{agarwal2017bandits} makes it possible to make optimized decisions, while still collecting data for unbiased estimation of people's probability of outcome, i.e., $p^*$, irrespective of the decision they received. While this involves considerable effort and collaboration with decision-makers, we believe that the twin benefits of an  interpretable single threshold and a straightforward learning algorithm that does not constrain use of the full data for modeling, outweigh the implementation costs.

Finally, we acknowledge that there will always be cases when unbiased data collection or active intervention is not possible. This is the case especially in scenarios where outcome data is available only for people who received one of the decisions, but never both. For example, we may observe outcomes for only the people who received a loan or who were let out on bail. As a result, it is hard to know the underlying selection biases that might have impacted the inclusion of people in a particular dataset, and we leave deriving unbiased estimators for these problems as an interesting direction for future work.

\subsubsection*{Acknowledgments.}  Arpita Biswas sincerely acknowledges the support of a Google PhD Fellowship Award. 
Siddharth Barman gratefully acknowledges the support of a Ramanujan Fellowship (SERB - {SB/S2/RJN-128/2015}) and a Pratiksha Trust Young Investigator Award. 

\bibliographystyle{alpha}
\bibliography{sample}

\end{document}